%% file: main.tex
\documentclass[sigconf,authorversion]{acmart}
\usepackage{xcolor}
\usepackage[utf8]{inputenc}
\usepackage{microtype}
\usepackage{algorithm}
\usepackage{subcaption}
\usepackage{amsthm}
\usepackage{amsmath}
\usepackage[capitalize]{cleveref}
\usepackage{bbold}
\usepackage{nicefrac}

\newtheorem{proposition}{Proposition}

\newcommand{\newterm}{\emph}
\newcommand{\dataset}{\mathcal{D}}

\newcommand{\sR}{\mathbb{R}}
\newcommand{\domain}{\mathcal{X}}
\newcommand{\classes}{\mathcal{Y}}
\newcommand{\train}{T}
\newcommand{\rashomon}{P_{\train(\dataset)}}
\newcommand{\predmult}{\mu}

\newcommand{\define}{~\triangleq~}
\newcommand{\id}{\mathbb{1}}

\newcommand{\E}{\mathbb{E}}

\newcommand{\hsiang}[1]{}
\newcommand{\bogdan}[1]{}
\newcommand{\flavio}[1]{}
\newcommand{\carmela}[1]{}

\title[Arbitrary Decisions are a Hidden Cost of Differentially Private Training]{Arbitrary Decisions are a Hidden Cost of \\Differentially Private Training}

\author{Bogdan Kulynych}
\affiliation{%
  \institution{EPFL SPRING Lab}
  \city{Lausanne}
  \country{Switzerland}
}

\author{Hsiang Hsu}
\affiliation{%
  \institution{Harvard University}
  \city{Cambridge}
  \state{MA}
  \country{USA}
}

\author{Carmela Troncoso}
\affiliation{%
  \institution{EPFL SPRING Lab}
  \city{Lausanne}
  \country{Switzerland}
}

\author{Flavio du Pin Calmon}
\affiliation{%
  \institution{Harvard University}
  \city{Cambridge}
  \state{MA}
  \country{USA}
}

\begin{abstract}
Mechanisms used in privacy-preserving machine learning often aim to guarantee differential privacy (DP) during model training. Practical DP-ensuring training methods use randomization when fitting model parameters to privacy-sensitive data (e.g., adding Gaussian noise to clipped gradients). We demonstrate that such randomization incurs predictive multiplicity: for a given input example, the output predicted by equally-private models depends on the randomness used in training. Thus, for a given input, the predicted output can vary drastically if a model is re-trained, even if the same training dataset is used. The predictive-multiplicity cost of DP training has not been studied, and is currently neither audited for nor communicated to model designers and stakeholders. We derive a bound on the number of re-trainings required to estimate predictive multiplicity reliably. We analyze—both theoretically and through extensive experiments—the predictive-multiplicity cost of three DP-ensuring algorithms: output perturbation, objective perturbation, and DP-SGD. We demonstrate that the degree of predictive multiplicity rises as the level of privacy increases, and is unevenly distributed across individuals and demographic groups in the data. Because randomness used to ensure DP during training explains predictions for some examples, our results highlight a fundamental challenge to the justifiability of decisions supported by differentially private models in high-stakes settings. We conclude that practitioners should audit the predictive multiplicity of their DP-ensuring algorithms before deploying them in applications of individual-level consequence.
\end{abstract}

\begin{document}

\begin{CCSXML}
<ccs2012>
   <concept>
       <concept_id>10010147.10010257</concept_id>
       <concept_desc>Computing methodologies~Machine learning</concept_desc>
       <concept_significance>500</concept_significance>
       </concept>
   <concept>
       <concept_id>10010147.10010257.10010293.10010294</concept_id>
       <concept_desc>Computing methodologies~Neural networks</concept_desc>
       <concept_significance>300</concept_significance>
       </concept>
   <concept>
       <concept_id>10010147.10010341.10010342.10010344</concept_id>
       <concept_desc>Computing methodologies~Model verification and validation</concept_desc>
       <concept_significance>500</concept_significance>
       </concept>
   <concept>
       <concept_id>10002978.10003029.10011150</concept_id>
       <concept_desc>Security and privacy~Privacy protections</concept_desc>
       <concept_significance>500</concept_significance>
       </concept>
   <concept>
       <concept_id>10002978.10003029.10003032</concept_id>
       <concept_desc>Security and privacy~Social aspects of security and privacy</concept_desc>
       <concept_significance>500</concept_significance>
       </concept>
 </ccs2012>
\end{CCSXML}

\ccsdesc[500]{Security and privacy~Privacy protections}
\ccsdesc[500]{Security and privacy~Social aspects of security and privacy}
\ccsdesc[500]{Computing methodologies~Model verification and validation}
\ccsdesc[500]{Computing methodologies~Machine learning}
\ccsdesc[300]{Computing methodologies~Neural networks}

\copyrightyear{2023}
\acmYear{2023}
\setcopyright{acmlicensed}\acmConference[FAccT '23]{2023 ACM Conference on Fairness, Accountability, and Transparency}{June 12--15, 2023}{Chicago, IL, USA}
\acmBooktitle{2023 ACM Conference on Fairness, Accountability, and Transparency (FAccT '23), June 12--15, 2023, Chicago, IL, USA}
\acmPrice{15.00}
\acmDOI{10.1145/3593013.3594103}
\acmISBN{979-8-4007-0192-4/23/06}

\maketitle

\section{Introduction}

\begin{figure*}
    \includegraphics[width=.9\linewidth]{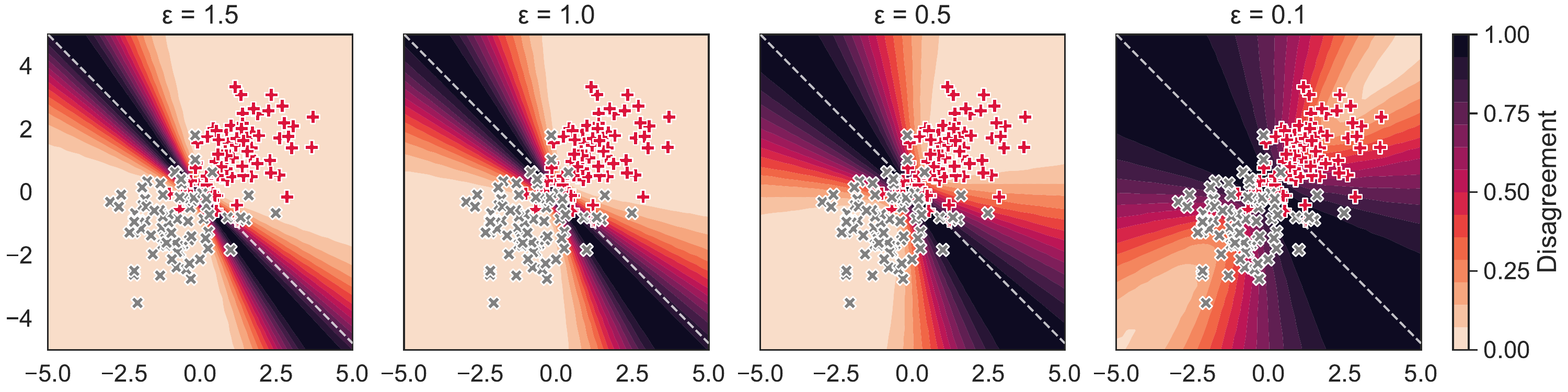}
    \caption{\textbf{The region of examples which exhibit high variance of decisions (dark) across similar models grows as the privacy level increases (lower $\varepsilon$).} Each plot shows the level of decision disagreement across $m = 5{\small,}000$ logistic-regression models (darker means higher disagreement) trained with varying levels of differential privacy ($\varepsilon$ value, lower means more private) using the objective-perturbation method~\cite{chaudhuri2011differentially}. All models attain at least 72\% accuracy on the test dataset (50\% is the baseline). The disagreement value of 1.0 means that out of the $m$ models, half output the positive decision, whereas the other half output the negative one for a given example. The values of disagreement are shown for different possible two-dimensional examples, with x and y axes corresponding to the two dimensions. The markers show training data examples belonging to two classes (denoted as $\textcolor{gray}{\bf \times}$ and $\textcolor{purple}{\bf +}$, respectively). Without DP, there is a single optimal classification model. The dotted line~\textcolor{gray}{-\,-} shows the decision boundary of this optimal non-private model. See \cref{sec:exp} for details.}
    \label{fig:syn-demo}
\end{figure*}

In many high-stakes prediction tasks (e.g., lending, healthcare), training data used to fit parameters of machine-learning models are privacy-sensitive. A standard technical approach to ensure privacy is to use training procedures that satisfy \newterm{differential privacy} (DP)~\cite{dwork2006calibrating, dwork2014algorithmic}. DP is a formal condition that, intuitively, guarantees a degree of plausible deniability on the inclusion of an individual sample in the training data. In order to satisfy this condition, non-trivial differentially-private training procedures use randomization (see, e.g., \citet{chaudhuri2011differentially,abadi2016deep}). The noisy nature of DP mechanisms is key to guarantee plausible deniability of a record's inclusion in the training data.
Unfortunately, randomization comes at a cost: it often leads to decreased accuracy compared to non-private training~\cite{jayaraman2019evaluating}. Reduced accuracy, however, is not the only cost incurred by differentially-private training. DP mechanisms can also increase \newterm{predictive multiplicity}, discussed next.

In a prediction task, there can exist multiple models that achieve comparable levels of accuracy yet output drastically different predictions for the same input. This phenomenon is known as predictive multiplicity~\cite{marx2020predictive}, and has been documented in multiple realistic machine-learning settings~\cite{marx2020predictive, hsu2022rashomon, watson2022predictive}. Predictive multiplicity can appear due to under-specification and randomness in the model's training procedure~\cite{d2020underspecification,black2021selective}.

Predictive multiplicity formalizes the \newterm{arbitrariness} of decisions based on a model's output. In practice, predictive multiplicity can lead to questions such as ``\emph{Why has a model issued a negative decision on an individual's loan application if other models with indistinguishable accuracy would have issued a positive decision?}'' or ``\emph{Why has a model suggested a high dose of a medicine for an individual if other models with comparable average accuracy would have prescribed a lower dose?}'' These examples highlight that acting on predictions of a single model without regard for predictive multiplicity can result in arbitrary decisions. Models produced by training algorithms that exhibit high predictive multiplicity face fundamental challenges to their credibility and justifiability in high-stakes settings ~\cite{d2020underspecification,black2022model}. 

In this paper, we demonstrate a fundamental connection between privacy and predictive multiplicity: For a fixed training dataset and model class,  DP training results in models that ensure the same degree of privacy and achieve comparable accuracy, yet assign conflicting outputs to individual inputs. DP training produces conflicting models even when non-private training results in a single optimal model. Thus, in addition to decreased accuracy, DP-ensuring training methods also incur an arbitrariness cost by exacerbating predictive multiplicity. We show that the degree of predictive multiplicity varies significantly across individuals and can disproportionately impact certain population groups. \cref{fig:syn-demo} illustrates the predictive-multiplicity cost of DP training in a simple synthetic scenario (see \cref{sec:exp} for examples on real-world datasets).

Our main contributions are:
\begin{enumerate}
    \item We provide the first analysis of the predictive-multiplicity cost of differentially-private training.
    \item We analyze a method for estimating the predictive-multi\-pli\-city properties of randomized machine-learning algorithms using re-training. We derive the first bound on the sample complexity of estimating predictive multiplicity with this approach. Our bound enables practitioners to determine the number of re-trainings required to estimate the predictive-multiplicity cost of randomized training algorithms up to a desired level of accuracy.
    \item We conduct a theoretical analysis of the predictive-multipli\-city cost of the output perturbation mechanism~\cite{chaudhuri2011differentially} used to obtain { a differentially-private logistic-regression model}. We characterize the exact dependence of predictive multiplicity on the level of privacy for this method.
    \item We conduct an empirical study of predictive multiplicity of two practical DP-ensuring learning algorithms: DP-SGD~\cite{abadi2016deep} and objective perturbation~\cite{chaudhuri2011differentially}. We use one synthetic dataset and five real-world datasets in the domains of finance, healthcare, and image classification. Our results confirm that, for these mechanisms, increasing the level of privacy invariably increases the level of predictive multiplicity. Moreover, we find that different examples exhibit different levels of predictive multiplicity. In particular, different demographic groups can have different average levels of predictive multiplicity.
\end{enumerate}

In summary, the level of privacy in DP training significantly impacts the level of predictive multiplicity. This, in turn, means that decisions supported by differentially-private models can have an increased level of arbitrariness: a given decision would have been different had we used a different random seed in training, even when all other aspects of training are kept fixed and the optimal non-private model is unique. Before deploying DP-ensuring models in high-stakes situations, we suggest that practitioners quantify predictive multiplicity of these models over salient populations and---if possible to do so without violating privacy---measure predictive multiplicity of individual decisions during model operation. Such audits can help practitioners evaluate whether the increase in privacy threatens the justifiability of decisions, choose whether to enact a decision based on a model's output, and determine whether to deploy a model in the first place.

\section{Technical Background}
\label{sec:background}

\subsection{Problem Setup and Notation} 
We consider a classification task on a \newterm{training dataset}, denoted as $\dataset \define \{(x_i, y_i)\}_{i=1}^n$, and consisting of examples $x_i \in \domain$ along with their respective labels $y_i \in \classes$.
In this work, we focus on the setting of binary classification, $\classes = \{0, 1\}$.
The goal of a classification task is to use the dataset to train a classifier $f_\theta: \domain\to\classes$, which accurately predicts labels for input examples in a given \newterm{test dataset} $\dataset_\mathrm{test} \in 2^{\domain \times \classes}$, where $2^{\domain \times \classes}$ denotes the power set over $\domain \times \classes$. 
Each classifier $f_\theta(x)$ is parameterized by a vector $\theta \in \Theta \subseteq \sR^d$.
{
A classifier associates a \newterm{confidence score} to each predicted input $x$, denoted as $h_\theta(x) \in [0, 1]$. 
If the confidence score is higher than some threshold $q \in [0, 1]$, then the decision is positive. Otherwise, it is negative. The classifier's prediction is thus obtained by applying a threshold to the confidence score: 
\begin{equation}\label{eq:threshold}
    f_\theta(x) \define \id[h_\theta(x) > q].
\end{equation}
In the rest of the paper, we use the standard threshold of $q = 0.5$.
}

We study randomized training algorithms $\train: (\domain\times\classes)^n \to\Theta$, which produce a parameter vector of a classifier in a randomized way. Thus, given a training dataset, $\train(\dataset)$ is a random variable. We denote by $\rashomon$ the \newterm{model distribution}, the probability distribution over $\Theta$ generated by the random variable $\train(\dataset)$.

In general, the source of randomness in the training procedure could include, e.g., random initializations of $\theta$ prior to training. However, we consider only those sources which are introduced by the privacy-preserving techniques, as we explain in the next section. Throughout this paper, the datasets, as well as any input example $x \in \domain$, are not random variables but fixed values. The only randomness we consider in our notation is due to the internal randomization of the training procedure $\train(\cdot)$. %
Finally, $I_d$ denotes the $d$-by-$d$ identity matrix, and $\id(\cdot)$ denotes the indicator function.

\subsection{Differentially Private Learning}
\label{sec:dp-bg}
Learning with differential privacy (DP) is one of the standard approaches to train models on privacy-sensitive data~\cite{dwork2006calibrating, dwork2014algorithmic}. 
A randomized learning algorithm $\train(\cdot)$ is $(\varepsilon, \delta)$-differentially private (DP) if for any two \newterm{neighbouring datasets} (i.e., datasets differing by at most one example) $\dataset, \dataset' \in (\domain \times \classes)^n$, for any subset of parameter vectors $A \subseteq \Theta$, it holds that
\begin{equation}
    \Pr[ \train(\dataset) \in A ] \leq \exp(\varepsilon) \Pr[ \train(\dataset') \in A] + \delta.
\end{equation}
In other words, the respective probability distributions of models produced on any two neighbouring datasets should be similar to a degree defined by parameters $(\varepsilon, \delta)$. The parameters represent the level of privacy: low $\varepsilon$ and low $\delta$ mean better privacy. DP mathematically encodes a notion of plausible deniability of the inclusion of an example in the dataset. 

There are multiple ways to ensure DP in machine learning. We describe next the \newterm{output perturbation mechanism}, which we theoretically analyze in \cref{sec:out-pert}.

\paragraph{Output perturbation mechanism \cite{rubinstein2012learning, chaudhuri2011differentially, wu2017bolt}.} Output perturbation is a simple mechanism for achieving DP that takes an output parameter vector of a non-private training procedure, and privatizes it by adding random noise, e.g., sampled from the isotropic Gaussian distribution. Concretely, suppose that $\train_\mathsf{np}: (\domain \times \classes)^n \rightarrow \Theta$ is a non-private learning algorithm. Denoting its output parameters as $\theta_\mathsf{np} = \train_\mathsf{np}(\dataset)$, we obtain the privatized parameters $\theta_\mathsf{priv} \in \Theta$ as:
\begin{equation}\label{eq:out-pert}
    \theta_\mathsf{priv} = \theta_\mathsf{np} + \xi, \text{ where } \xi \sim \mathcal{N}(0, \sigma^2 I_d)
\end{equation}
The exact level of DP provided by this procedure depends on the choice the non-private training algorithm $\train_\mathsf{np}(\dataset)$. In particular, to achieve $(\varepsilon, \delta)$-DP, it is sufficient to set the noise scale $\sigma = C \cdot \nicefrac{\sqrt{2 \log(1.25 / \delta)}}{\varepsilon}$,
where $C \define \max_{\dataset \sim \dataset'} \| \train_\mathrm{np}(\dataset) - \train_\mathrm{np}(\dataset') \|_2$ is the \newterm{sensitivity} of the non-private training algorithm, the maximum discrepancy in terms of the $\ell_2$ distance between parameter vectors obtained by training on any two neighbouring datasets $\dataset, \dataset'$.

Denoting by $\train(\dataset) = \train_\mathrm{np}(\dataset) + \xi$ the output-perturbation procedure in \cref{eq:out-pert}, we treat $\train(\dataset)$ as a random variable over the randomness of the injected noise $\xi$. Other methods to achieve DP such as objective perturbation~\cite{chaudhuri2011differentially} also inject noise as part of training. In those cases, we similarly consider $\train(\dataset)$ as a random variable over such injected noise, and treat all other aspects of training such as pre-training initialization as fixed.

\subsection{Predictive Multiplicity}

Predictive multiplicity occurs when multiple classification models achieve comparable average accuracy yet produce conflicting predictions on a given example \cite{marx2020predictive}. To quantify predictive multiplicity in randomized training, we need to measure dissimilarity of predictions among the models sampled from the probability distribution $\rashomon$ induced by differentially-private training. For this, we use a definition of \newterm{disagreement} which has appeared in different forms in \cite{black2022model,d2020underspecification,marx2020predictive}. For a given fixed input example $x \in \domain$, we define the disagreement $\predmult(x)$ as:
\begin{equation}\label{def:disagreement}
    \predmult(x) \define 2 \Pr_{\theta, \theta' \sim \rashomon}[f_{\theta}(x) \neq f_{\theta'}(x)].
\end{equation}
In the above definition, $\theta, \theta' \sim \rashomon$ denotes two models sampled independently from $\rashomon$. We use a scaling factor of two in order to ensure that $\predmult(x)$ is in the $[0, 1]$ range for the ease of interpretation. A disagreement value $\mu(x)\approx 1$ indicates that the prediction for $x$ is approximately equal to an unbiased coin flip. Moreover, a disagreement $\mu(x)\approx 0$ implies that, with high probability, the prediction for $x$ does not significantly change if two models are independently sampled from $\rashomon$ (i.e., by re-training a model twice with different random seeds).

{ In the literature, a commonly studied source of variance of outcomes of training algorithms is from re-sampling of the dataset $\dataset$, usually under the assumption that it is an i.i.d. sample from some data distribution. We do not study variance arising from dataset re-sampling, and are only interested in the predictive-multiplicity properties of the randomized training procedure $\train(\cdot)$ itself.
Thus, we \emph{fix} both the dataset $\dataset$ used in training and the input example $x$ for which we compute the level of predictive multiplicity, and make sure that the randomness is only due to internal randomization of the training procedure $\train(\cdot)$. }

When evaluating dissimilarity across models, many prior works that study predictive multiplicity (e.g.,~\cite{marx2020predictive, semenova2022existence, hsu2022rashomon, watson2022predictive}) only consider models that surpass a certain accuracy threshold. 
Although conditioning on model accuracy is theoretically valid, it can bring about confusion in the context of private learning, as in practice such conditioning would demand special mechanisms in order to satisfy DP (see, e.g., \cite{papernot2021hyperparameter}). In particular, first applying a DP training method that guarantees an $(\varepsilon,\delta)$-level of privacy, and then selecting or discarding the resulting model based on accuracy, would result in models that violate the initial $(\varepsilon,\delta)$-DP guarantees.  
We note, however, that our results and experiments involving estimation of predictive multiplicity in \cref{sec:estimator,sec:exp} extend to the case in which we add  additional conditioning on top of  model distribution $\rashomon$ to control for accuracy.

Before proceeding with our analyses of disagreement, we first state a simple yet useful relation between disagreement and statistical variance. Observe that for a given input $x$, the output prediction $f_\theta(x)$ is a random variable over the randomness of the training procedure $\theta \sim \rashomon$. As we assume that the decisions are binary, and training runs are independent, we have that $f_\theta(x) \sim \mathrm{Bernoulli}(p_x)$ for some input-specific parameter $p_x$.
Having noted this fact, we show that disagreement, defined in \cref{def:disagreement}, can be expressed as a continuous transformation of $p_x$:
\begin{proposition}{multvar}\label{stmt:mult-var}
For binary classifiers, disagreement for a given example $x \in \domain$ is proportional to variance of decisions over the distribution of models generated by the training algorithm:
\begin{equation}
    \predmult(x) = 4 \, \mathrm{Var}_{\theta \sim \rashomon}(f_\theta(x)) = 4p_x \, (1-p_x).
\end{equation}
\end{proposition}
We provide the proof of this and all the following formal statements in \cref{app:proofs}. { Additionally, in \cref{sec:exp-details}, we provide an analysis using an alternative measure of predictive multiplicity.}

\section{Predictive Multiplicity of Output Perturbation}
\label{sec:out-pert}

\begin{figure*}
    \centering
    \includegraphics[width=.35\linewidth]{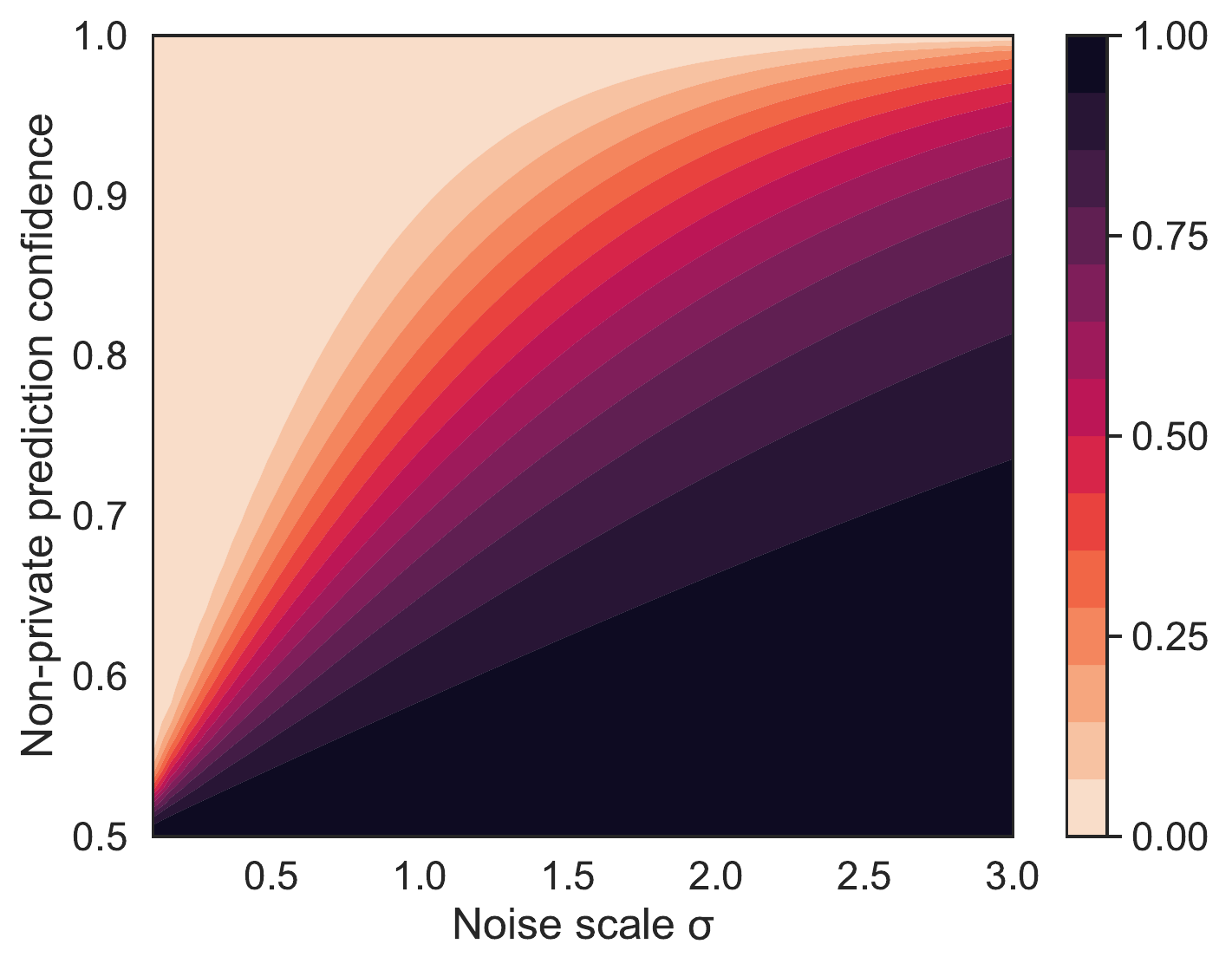}~
    \includegraphics[width=.33\linewidth]{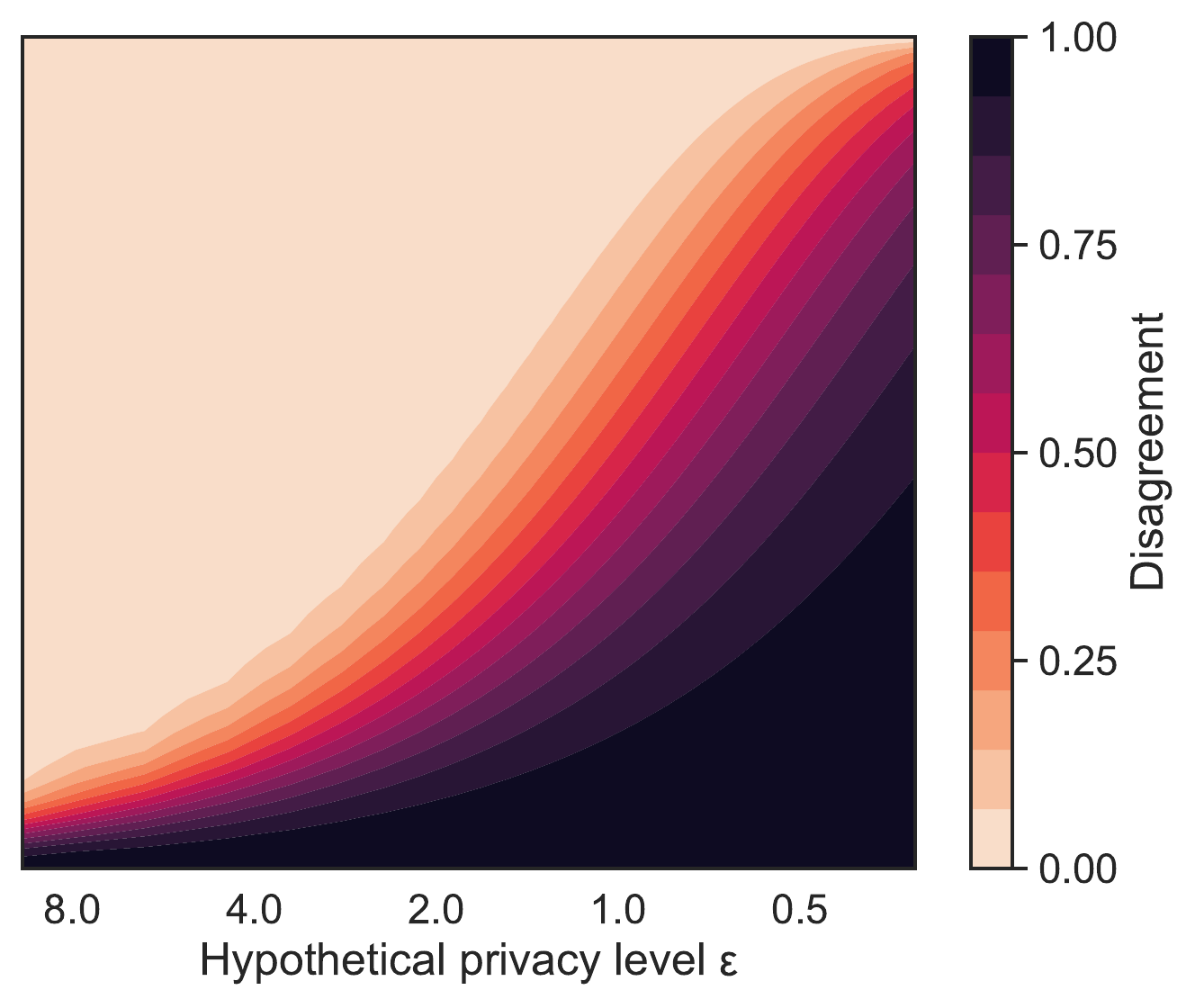}
    \caption{\textbf{The noise scale in output perturbation mechanisms increases predictive multiplicity for examples which do not attain high non-private prediction confidence.} On the left, the x axis shows the noise scale used for output perturbation (higher values of $\sigma$ correspond to better privacy). The noise scale corresponds to different levels of privacy depending on the sensitivity of the non-private training algorithm and the $\delta$ parameter (see \cref{sec:dp-bg}). On the right, the x axis (logarithmic scale) shows a possible level of privacy $\varepsilon$ for $\delta = 10^{-5}$, assuming that the non-private training algorithm has sensitivity of $C = 0.2$. The y axis shows the hypothetical prediction confidence for a given example. The color intensity shows the level of disagreement (darker means higher disagreement).}
    \label{fig:out-pert}
\end{figure*}

To demonstrate how DP training can lead to an increase in predictive multiplicity, we theoretically analyze the multiplicity properties of the output-perturbation mechanism described in \cref{sec:dp-bg}.

Following \citet{chaudhuri2011differentially} and~\citet{wu2017bolt}, we study the case of logistic regression. { In a logistic-regression model parameterized by vector $\theta \in \sR^d$, we compute the confidence score for an input $x \in \domain \subseteq \sR^d$ as $h_\theta(x) = \mathsf{sigmoid}(\theta^\intercal x)$, where
\begin{equation}
    \mathsf{sigmoid}(t) \define \frac{1}{1 + \exp(-t)}.
\end{equation}
Recall that the classifier's prediction is obtained by applying a threshold to the confidence score by \cref{eq:threshold}, in this case as $f_\theta(x) = \id[\mathsf{sigmoid}(\theta^\intercal x) > 0.5]$.} Note that the quantity $\theta^\intercal x$ is interchangeable with confidence, as one can be obtained from the other using an invertible transformation. We show the exact relationship between disagreement and the scale of noise $\sigma$ in this setting:

\begin{proposition}
\label{prop:outpert}
Let $\theta_\mathsf{np} = \train_\mathsf{np}(\dataset)$ be a non-private parameter vector of a logistic-regression model. Suppose that the privatized $\theta_\mathsf{priv}$ is obtained using Gaussian noise of scale $\sigma$ as in \cref{eq:out-pert}. Then, the disagreement of a private logistic-regression model parameterized by $\theta_\mathsf{priv}$ is:
\begin{equation}\label{eq:mult-out-pert}
    \predmult(x) = 4 \, p_x (1 - p_x), \text{where $p_x = \Phi\left(\frac{\theta_\mathrm{np}^\intercal \, x}{\|x\| \cdot \sigma}\right)$.}
\end{equation}
\end{proposition}

We visualize the relationship in \cref{fig:out-pert}, assuming the input space is normalized so that $\|x\| = 1$. There are two main takeaways from this result. First, disagreement is high when the level of privacy is high. Second,  the level of multiplicity is unevenly distributed across input examples. This is because the exact relationship between multiplicity and privacy also depends on the confidence of the non-private model, $\theta_\mathsf{np}^\intercal \, x$, with lower-confidence examples generally having higher multiplicity in this setting. We note that, in this illustration, the simple relationship between confidence and predictive multiplicity is an artifact of normalized features, i.e.,  $\|x\| = 1$. In general, examples with high-confidence predictions can display high predictive multiplicity after DP-ensuring training, as illustrated in \cref{sec:exp-main}.

Other methods for DP training, such as gradient perturbation~\cite{abadi2016deep}, are not as straightforward to analyze theoretically. In the next sections, we study predictive multiplicity of these algorithms using a Monte-Carlo method.

\section{Measuring Predictive Multiplicity of Randomized Algorithms}
\label{sec:estimator}
Theoretically characterizing predictive multiplicity of DP algorithms beyond the output-perturbation mechanism and for more complex model classes is a challenging problem (see, e.g.~\cite[Section~4]{hsu2022rashomon}). For instance, the accuracy and generalization behavior of the DP-SGD algorithm \cite{abadi2016deep} used for DP training of neural networks is an active area of research (e.g., \cite{wang2023generalization}). Even in simpler model classes, where training amounts to solving a convex optimization problem (e.g., support vector machines), DP mechanisms such as objective perturbation \cite{chaudhuri2011differentially} display a complex interplay between privacy, accuracy, and distortion of model parameters.

For these theoretically intractable cases, we adopt a simple { Monte-Carlo strategy~\cite{d2020underspecification,black2021selective}}: Train multiple models on the same dataset with different randomization seeds, and compute statistics of the outputs of these models. { Note that this procedure does not preserve differential privacy, which we discuss in more detail in \cref{sec:open}.}

In this section, we formalize this simple and intuitive approach, and provide the first sample complexity bound for estimating predictive multiplicity. Our bound has a closed-form expression, so a practitioner can use it to determine how many re-trainings are required  to estimate predictive multiplicity up to a given approximation error. 

At first, re-training might appear as a blunt approach for analyzing predictive multiplicity in DP. Our results indicate that this is not the case.  Surprisingly, we prove that, if one wants to estimate disagreement in \cref{def:disagreement} for $k$ input examples, the number of required re-trainings increases \emph{logarithmically} in $k$.  This result demonstrates that re-training can be an effective strategy to estimate predictive multiplicity regardless of the intricacies of a specific DP mechanism, and that a moderate number of re-trainings is sufficient to estimate disagreement for a large number of examples.

Recall that, according to \cref{stmt:mult-var}, disagreement of an example $x$ is proportional to the variance of outputs within the model distribution $\rashomon$. We use this connection to provide an unbiased estimator for disagreement.
\begin{proposition}
\label{prop:unbiasedest}
Suppose we have $m$ models sampled from the model distribution: $\theta_1, \theta_2, \ldots, \theta_m \sim \rashomon$. Then, the following expression is an unbiased estimator for disagreement $\predmult(x)$ for a single example $x \in \domain$:
\begin{equation}\label{eq:disagreement-estimator}
\hat \predmult(x) \define 4 \, \frac{m}{m-1} \, \hat p_x(1 - \hat p_x),
\end{equation}
where $\hat p_x = \frac{1}{m} \sum_{i = 1}^m f_{\theta_i}(x)$ is the sample mean of $f_\theta(x)$.
\end{proposition}

How many models $\theta_1, \theta_2, \ldots, \theta_m$ do we need to sample in order to estimate disagreement?
To answer this, we provide an upper bound on estimation accuracy given the number of samples from the model distribution, as well as a bound on the number of samples required for a given level of estimation accuracy.
\begin{proposition}
\label{prop:single}
    For $m$ models sampled from the model distribution, $\theta_1, \theta_2, \ldots, \theta_m \sim P_{\train(\dataset)}$, with probability at least $1 - \rho$, for $\rho \in (0, 1]$ the additive estimation error $\alpha \define |\hat \predmult(x)-\predmult(x)|$ satisfies:
    \begin{equation}
    \alpha \leq \frac{1}{(m-1)} + 4 \frac{m}{m-1}\sqrt{\frac{\log(2/\rho)}{2m}}\left(1+\sqrt{\frac{\log(2/\rho)}{2m}}\right).
    \end{equation}
\end{proposition}

For example, this bound yields that 5,000 re-trainings result in the estimation error of at most $0.08$ with probability $95\%$. In \cref{app:estimator}, we provide a closed-form expression for computing the number of samples $m$ required to achieve a given error level $\alpha$. We also provide a visualization of the bound in \cref{fig:estimator-rate} (Appendix).

In practice, one might need to estimate disagreement for multiple examples, e.g., to compute average disagreement over a test dataset. When doing so na\"ively, the re-training costs could mount to infeasible levels if we assume that each estimation requires the same number of models, $m$, for each input example. In contrast, we show that in such cases sample complexity grows only logarithmically.

\begin{proposition}\label{prop:multi}
Let $x_1, x_2, \ldots, x_k \in \domain$. 
If $\theta_1, \theta_2, \ldots, \theta_m \sim \rashomon$ are i.i.d. samples from the model distribution, then with probability at least $1 - \rho$, for $\rho \in (0, 1]$ the maximum additive error satisfies:
\begin{equation}
    \begin{aligned}
    \max_{j \in 1, \ldots, k} |\predmult(x_j) - \hat \predmult(x_j)| &\leq \frac{1}{(m-1)} + \\
        &+ \frac{4m}{m-1}\sqrt{ \frac{\log (2k/\rho)}{2m}}\left(1+\sqrt{\frac{\log (2k/\rho)}{2m}}\right).
    \end{aligned}
\end{equation}
\end{proposition}
This positive result shows that auditing models for predictive multiplicity for large populations and datasets is practical, as the sample complexity grows slowly in the number of examples.

\section{Empirical Studies}
\label{sec:exp}
In this section, we empirically explore the predictive multiplicity  of DP algorithms. We use a low-dimensional synthetic dataset in order to visualize the level of multiplicity across the input space. To study predictive-multiplicity effects in realistic settings, we use real-world tabular datasets representative of high-stakes domains, namely lending and healthcare, and one image dataset.
The code to reproduce our experiments is available at:
\begin{center}
\href{https://github.com/spring-epfl/dp_multiplicity}{github.com/spring-epfl/dp\_multiplicity}
\end{center}

\subsection{Experimental Setup}

\paragraph{Datasets and Tasks}

We use the following datasets: 
\begin{itemize}
\item A \textbf{Synthetic} dataset containing data belonging to two classes with class-conditional distributions $X_0 \sim \mathcal{N}(\mu_0, \Sigma_0)$ and $X_1 \sim \mathcal{N}(\mu_1, \Sigma_1)$, respectively. We set the distribution parameters to be:
\begin{equation}
    \begin{aligned}
    &\mu_0 = [1, 1], &\Sigma_0 = \begin{pmatrix}1 & \nicefrac{1}{2} \\ \nicefrac{1}{2} & 1\end{pmatrix},
    \\
    &\mu_1 = [-1, -1], &\Sigma_1 = \begin{pmatrix}1 & \nicefrac{1}{10} \\ \nicefrac{1}{10} & 1\end{pmatrix}.
    \end{aligned}
\end{equation}
The classes in this synthetic dataset are well-separable by a linear model (see \cref{fig:syn-demo}) %

    \item Credit Approval tabular dataset (\textbf{Credit}). The task is to predict whether a credit card application should be approved or rejected based on several attributes which describe the application and the applicant.
    \item Contraceptive Method Choice tabular dataset (\textbf{Contracep\-tion}) based on 1987 National Indonesia Contraceptive Prevalence Survey. The task is to predict the choice of a contraception method based on demographic and socio-economic characteristics of a married couple.
    \item Mammographic Mass tabular dataset (\textbf{Mammography}) collected at the Institute of Radiology of the University Erlangen-Nuremberg in 2003 -- 2006. The task is to predict whether a screened tumor is malignant or benign based on several clinical attributes.
    \item \textbf{Dermatology} tabular dataset. The task is to predict a dermatological disease based on a set of clinical and histopathological attributes.
    \item \textbf{CIFAR-10}~\cite{krizhevsky2009learning}, an \emph{image} dataset of pictures labeled as one of ten classes. The task is to predict the class.
\end{itemize}
We take the realistic tabular datasets (Credit, Contraception, Mammography, and Dermatology) from the University of California Irvine Machine Learning (UCIML) dataset repository~\cite{uciml}. In \cref{sec:exp-details}, we provide additional details about processing of the datasets, and a summary of their characteristics (\cref{tab:datasets}).

For the synthetic dataset, we obtain the training dataset by sampling 1,000 examples from each of the distributions. In order to have precise estimates of population accuracy, we sample a larger test dataset of 20,000 examples. For tabular datasets, we use a random $75\%$ subset for training, and use the rest as a held-out test dataset for model evaluations. For CIFAR-10, we use the default 50K/10K train-test split.

\paragraph{Models and Training Algorithms}
For the synthetic and tabular datasets, we use logistic regression with objective perturbation~\cite{chaudhuri2011differentially}.
{ For the image dataset, we train a convolutional neural network on ScatterNet features~\cite{oyallon2015deep} using DP-SGD~\cite{abadi2016deep}, following the approach by \citet{tramer2021differentially}. We provide more details in \cref{sec:exp-details}. }

\paragraph{Metrics} The goal of our experiments is to quantify predictive multiplicity and explain the factors which impact it. For all settings, we measure disagreement to capture the dissimilarity of predictions, and predictive performance of the models to quantify the effect of performance on multiplicity. Concretely, we measure:
\begin{itemize}
    \item \textbf{Disagreement} for examples on a test dataset, computed using the unbiased estimator in \cref{sec:estimator}. As this disagreement metric is tailored to binary classification, we use a special procedure for the ten-class task on CIFAR-10: we treat each multi-class classifier as ten binary classifiers, and we report average disagreement across those ten per-class classifiers. { Additionally, in \cref{sec:exp-details}, we also report predictive multiplicity in terms of confidence scores instead of predictions following the recent approach by \citet{watson2022predictive}. }
    \item { \textbf{Performance} on a test dataset. For tabular datasets, we report the standard area under the ROC Curve (AUC for short). For CIFAR-10, we report accuracy.}
\end{itemize}

\paragraph{Experiment Outline}
For a given dataset and a value of the privacy parameter $\varepsilon$, we train multiple models on exactly the same data \textit{with different randomization seeds}. %

For the synthetic and tabular datasets, we use several values of $\varepsilon$ between 0.5 (which provides a good guaranteed level of privacy~\cite[see, e.g.][Section 4]{wood2018differential}) and 2.5, with $\delta = 0$. For each value of $\varepsilon$ we train $m = 5{\small,}000$ models. 
For CIFAR-10, we train $m = 50$ neural-network models because of computational constraints. We use DP-SGD parameters that provide privacy guarantees from $\varepsilon \approx 2$ to $\varepsilon \approx 7$ at the standard choice of $\delta = 10^{-5}$.

\begin{figure*}
    \centering
    \begin{subfigure}[t]{.7\textwidth}
        \hspace{1cm}
        \includegraphics[width=\linewidth]{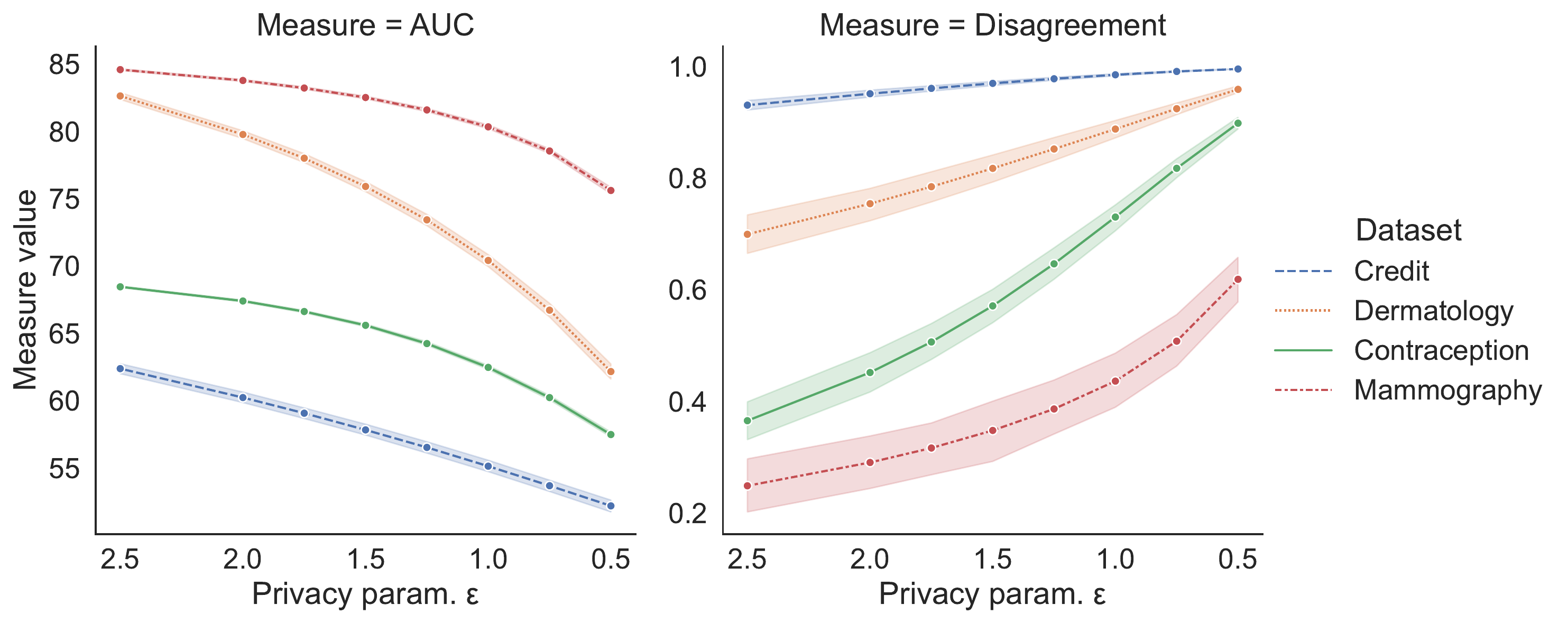}
        \vspace{-1.2em}
        \caption{Tabular datasets}
    \end{subfigure}
    \\[.5em]
    \begin{subfigure}[t]{.58\textwidth}
        \includegraphics[width=\linewidth]{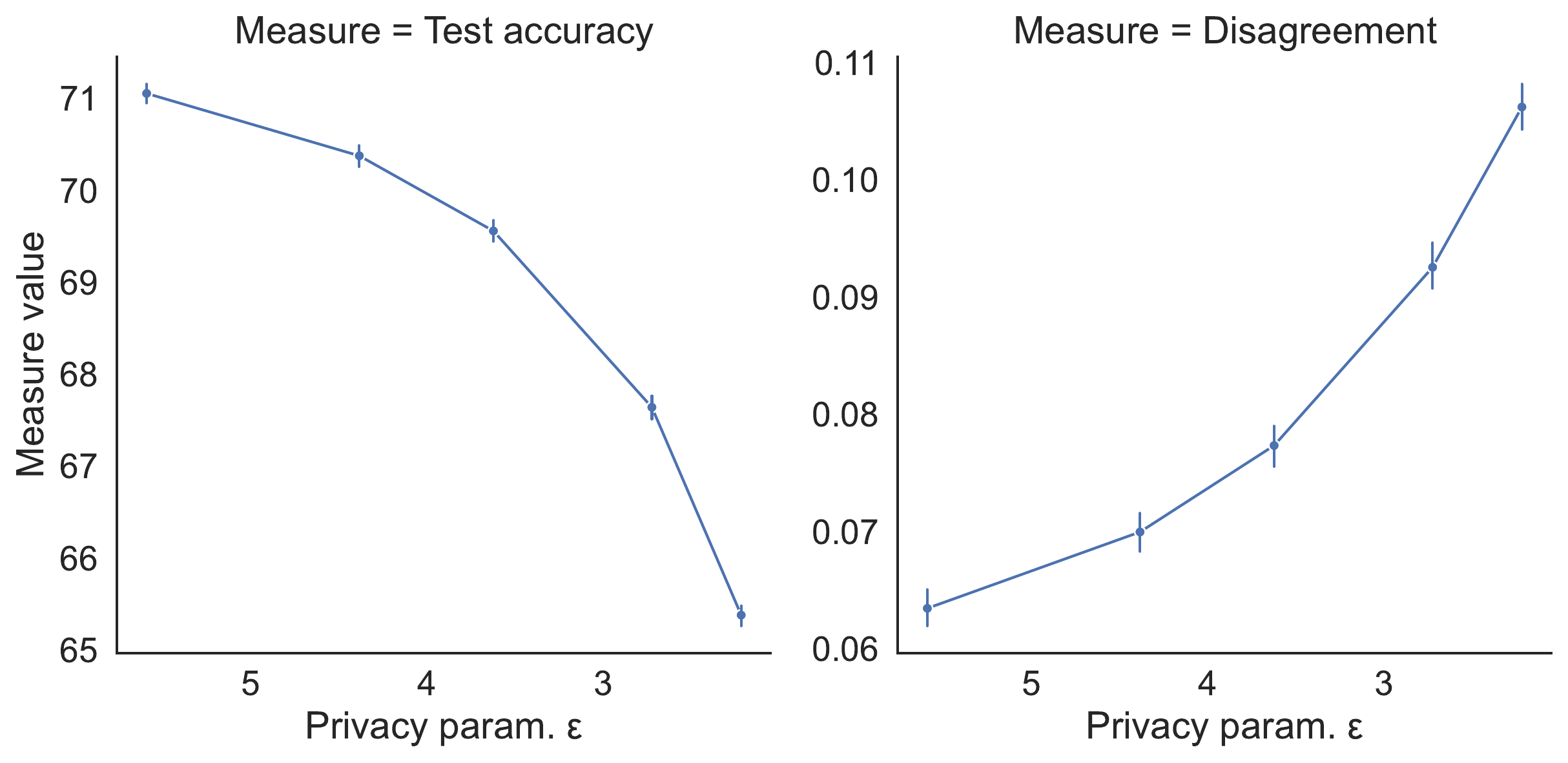}
        \vspace{-1.2em}
        \caption{Image dataset (CIFAR-10)}
    \end{subfigure}
    
    \caption{\textbf{Increasing the level of privacy increases the level of predictive multiplicity in real-world datasets.} For all plots, the x axis shows the level of privacy ($\varepsilon$, lower value is more privacy). The plots on the left shows the performance level (AUC for tabular datasets, and accuracy for CIFAR-10). The error bands/bars on the left side are 95\% confidence intervals (CI) over the models in the model distribution. The plots on the right show the degree of disagreement across $m = 5{\small,}000$ models in the case of tabular datasets, and across $m = 50$ models in the case of CIFAR-10. The error bands/bars on the right side are 95\% CI over the examples in a test dataset.  Although average disagreement might be relatively low for some datasets such as Mammography and CIFAR-10, there exist examples for which disagreement is 100\% (see \cref{tab:exp-stats} in the Appendix).}
    \label{fig:mult-main}
\end{figure*}

\subsection{Predictive Multiplicity and Privacy}
\label{sec:exp-main}

First, we empirically study how multiplicity evolves with increasing privacy. In \cref{fig:syn-demo}, we visualize the two-dimensional synthetic examples and their disagreement for different privacy levels.
As privacy increases, so do the areas for which model decisions exhibit high disagreement (darker areas).
Although the regions with higher disagreement correlate with model confidence and accuracy, the level of privacy contributes significantly. For instance, some points which are relatively far from the decision boundary, which means they are confidently classified as either class, can nevertheless have high predictive multiplicity.

\cref{fig:mult-main} shows the experimental results for our tabular datasets and CIFAR-10. On the left side, we show the relationship between the privacy level and performance. On the right, between the privacy level and disagreement. As with the theoretical analysis and the results on synthetic data, we can clearly see that models with higher level of privacy (low $\varepsilon$) invariably exhibit higher predictive multiplicity. { Notably, even for datasets such as Mammography and CIFAR-10 for which \emph{average} disagreement is relatively low, there exist examples whose disagreement is 100\%. See \cref{tab:exp-stats} in the Appendix for detailed information on the distribution of the disagreement values across the test data.}

\vspace{1mm}\noindent\textbf{Implications.} The increase in the privacy level results in making more decisions which are partially or fully explained by randomness in training. { Let us give an example with a concrete data record from the Mammography dataset representing a 56-year-old patient labeled as having a malignant tumor. 
Classifiers with low level of privacy $\varepsilon=2.5$ predict the correct malignant class for this individual most of the time (approx. 55\% disagreement). If we set the level of privacy to the high $\varepsilon=0.5$, this record is classified close to 42\% of the time as benign, and 58\% of the time as malignant (approx. 97\% disagreement). Thus, if one were to use a model with the high level of privacy to inform treatment of this patient, the model's decision would have been close in its utility to a coin flip.}

\subsection{What Causes the Increase in Predictive Multiplicity?}
In the previous section, we showed that the increase in privacy causes an increase in predictive multiplicity. It is not clear, however, what is the exact mechanism through which DP impacts predictive multiplicity. Hypothetically, the contribution to multiplicity could be through two pathways:
\begin{enumerate}
    \item[(1)] \emph{Direct:} The increase in predictive multiplicity is the result of the variability in the learning process stemming from randomization, regardless of the performance decrease.
    \item[(2)] \emph{Indirect}: The increase in predictive multiplicity is the result of the decrease in performance.
\end{enumerate}

\begin{figure}
    \centering
    \includegraphics[width=.8\linewidth]{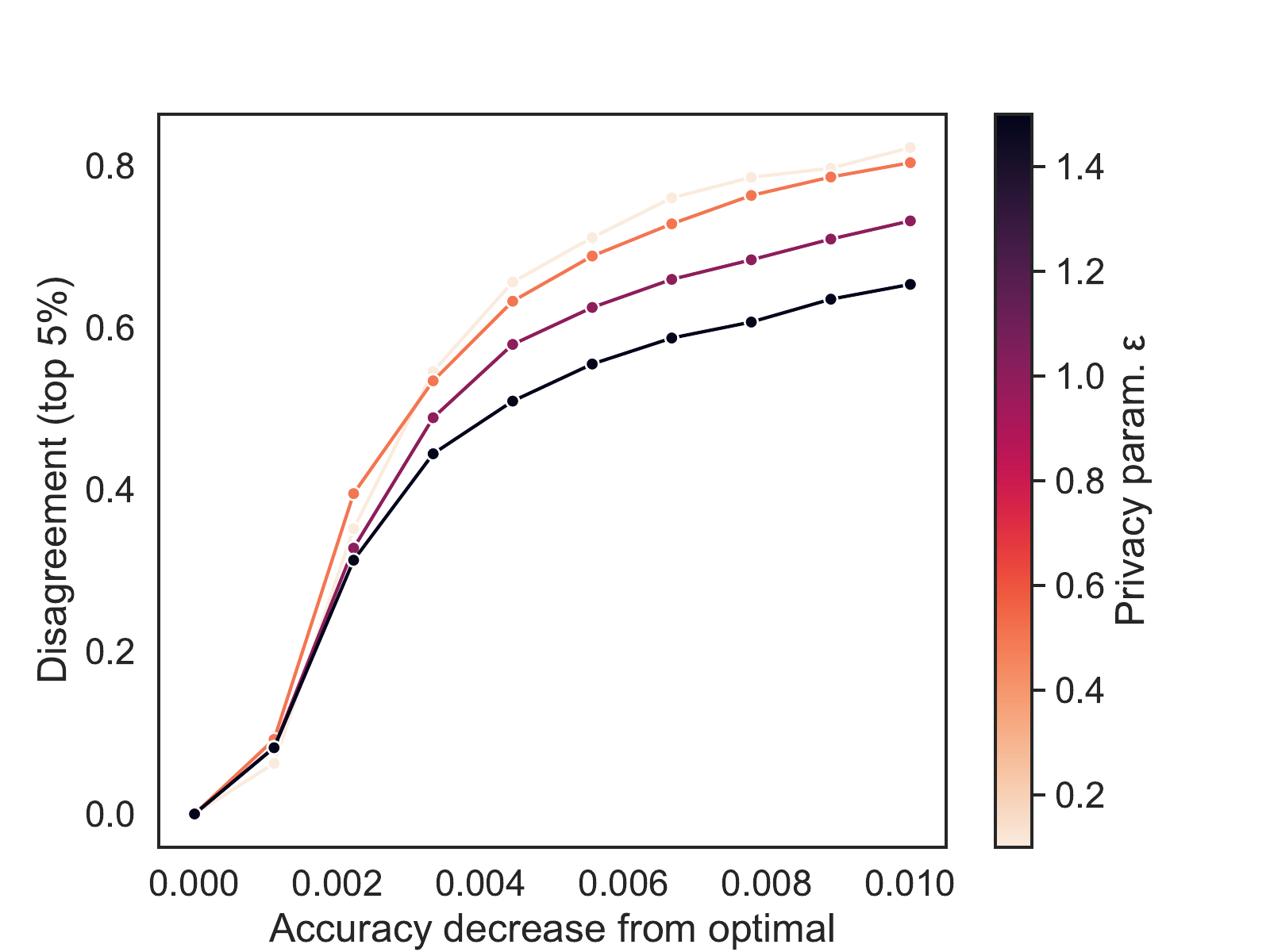}
    \caption{\textbf{Models achieving a similar level of accuracy can have different levels of predictive multiplicity.} The plot shows the top 5\% percentile of disagreement on the synthetic test dataset for all models which attain at least certain level of accuracy, for different values of the privacy parameter ($\varepsilon$, lower value is more privacy). The x axis shows the deviation of accuracy from that of an optimal non-private model, with 0 being equal to the accuracy of the optimal non-private model. As even such a small decrease in accuracy as 0.01 can see disagreement rise from 0 to 0.8 for some examples, this result suggests that the change in the level of privacy on its own can cause a big change in disagreement. %
    }
    \label{fig:syn-mult-acc}
\end{figure}

These two options are not mutually exclusive, and it is possible that both play a role.
In both cases, the desire for a given level of privacy---which determines the degree of  randomization added during training---is ultimately the cause of the increase in multiplicity. Nevertheless, how randomization contributes to the increase has practical implications: If our results are explained by pathway (2), we should be able to reduce the impact of privacy on predictive multiplicity by designing algorithms which achieve better accuracy at the same privacy level. 

For output perturbation, our analysis in \cref{sec:out-pert} shows that multiplicity is directly caused by randomization---path\-way (1)---as only the privacy level, confidence, and the norm of a predicted example impact disagreement. Therefore, performance does not have a direct impact on predictive multiplicity in output perturbation.

In \cref{fig:syn-mult-acc}, to quantify the impact of performance on predictive multiplicity for the case of objective perturbation, we show the top 5\% disagreement values for varying levels of accuracy on the synthetic dataset. We use the synthetic dataset to ensure that test accuracy estimates are reliable, as we have a large test dataset in this case.
We see that, for a given level of accuracy, different privacy parameters can result in different disagreement. This suggests that randomization caused by  DP training \textit{can have a direct effect} on predictive multiplicity, so we observe pathway (1).

\vspace{1mm}\noindent\textbf{Implications.} This observation indicates that there exist cases for which improving accuracy of a DP-ensuring algorithm at a given privacy level \textit{will not} necessarily lower predictive multiplicity.

\subsection{Disparities in Predictive Multiplicity}
\label{sec:exp-disparities}
The visualizations in \cref{fig:syn-demo} show that different examples can exhibit highly varying levels of predictive multiplicity. 
This observation holds for real-world datasets too. \cref{fig:tabular-disparities-dist} shows the distributions of the disagreement values across the population of examples in the test data for tabular datasets. For example, for lower privacy levels (high $\varepsilon$) on the Contraception dataset, there are groups of individuals with different values of predictive multiplicity. As the level of privacy increases (low $\varepsilon$), the disagreement tends to concentrate around 1, with decisions for a majority of examples largely explained by randomness in training.

Next, we verify if the differences in the level of disagreement also exist across demographic groups. In \cref{fig:tabular-disparities-group}, we show average disagreement across points from three different age groups in the Contraception dataset. As before, for low levels of privacy (high $\varepsilon$) we see more disparity in disagreement. The disparities even out as we increase the privacy level (low $\varepsilon$), with groups having average disagreement closer to 1. Thus, disagreement is not only unevenly distributed across individuals, but across salient demographic groups.

{
\vspace{1mm}\noindent\textbf{Implications.} As some groups and individuals can have higher predictive multiplicity than others, evaluations of training algorithms in terms of their predictive multiplicity must account for such disparities. For instance, our experiments on the Contraception dataset (in \cref{fig:tabular-disparities-group}) show that, for different privacy levels, decisions for individuals in the 16--30 age bracket exhibit higher predictive multiplicity than of patients between 30 and 40 years old. Predictions for individuals under 30, therefore, systematically exhibit more dependence on randomness in training than on the relevant features for prediction. This highlights the need to conduct disaggregated evaluations as opposed to only evaluating average disagreement on whole datasets.}

\begin{figure*}
    \centering
    \begin{subfigure}[t]{1.0\textwidth}
        \centering
        \includegraphics[width=1.0\linewidth]{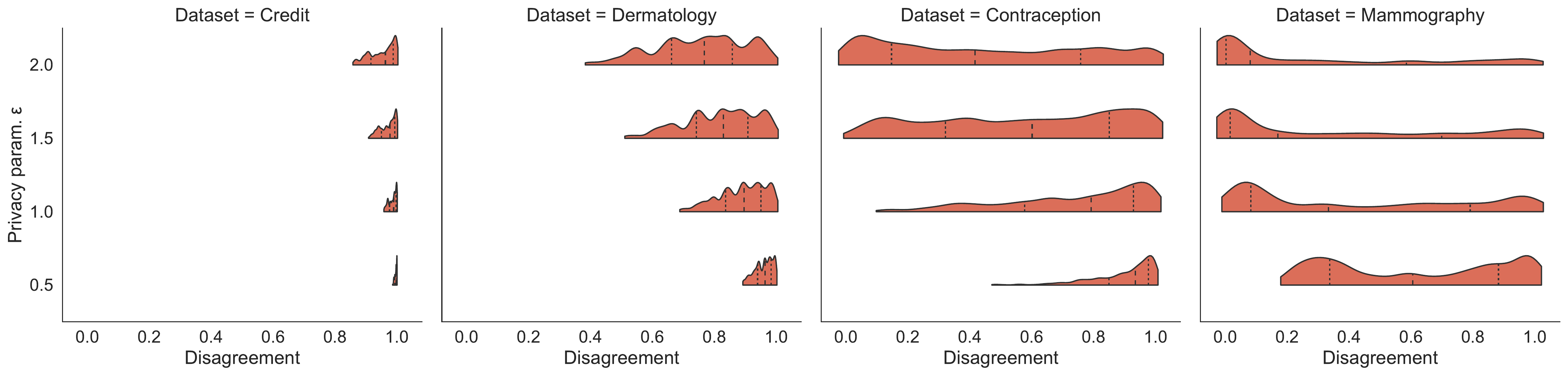}
        \caption{Distribution of disagreement values across the population in the test set in four tabular datasets.}
        \label{fig:tabular-disparities-dist}
        \vspace{1em}
    \end{subfigure}
    \begin{subfigure}[t]{1\textwidth}
        \centering
        \vspace{1em}
        \includegraphics[width=.84\linewidth]{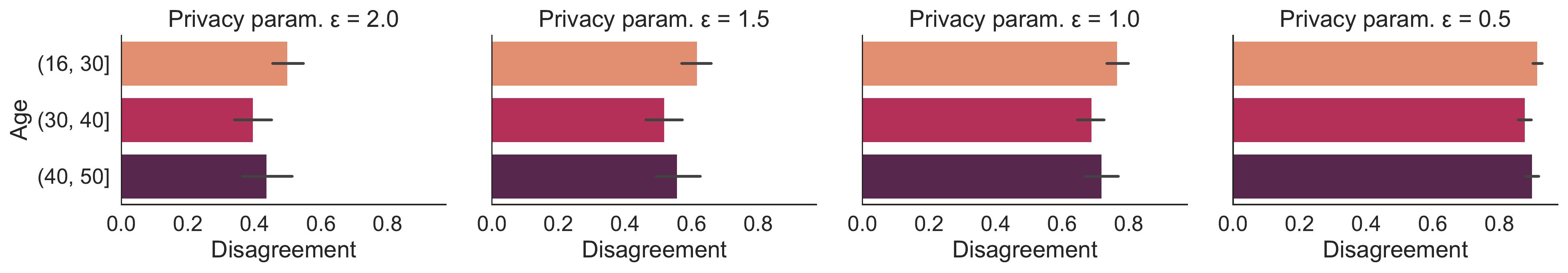}
        \caption{Group-level disparities in disagreement values on the Contraception dataset. Error bars are 95\% confidence intervals over the disagreement values in each subgroup.}
        \label{fig:tabular-disparities-group}
    \end{subfigure}
    
    \caption{\textbf{The level of predictive multiplicity varies from one example to another, and across population groups. As the level of privacy grows, more predictions exhibit similarly high disagreement.}
    }
    \label{fig:tabular-disparities}
\end{figure*}

\section{Related work}\label{sec:related-work}
\paragraph{Rashomon Effect and Predictive Multiplicity.}
The Rashomon effect, observed and termed by \citet{breiman2001statistical}, describes the phenomenon where a multitude of distinct models achieve similar average loss. The Rashomon effect occurs even for simple models such as linear regression, decision trees, and shallow neural networks \cite{auer1995exponentially}. 
When no privacy constraints are present, predictive multiplicity can be viewed as a facet of the Rashomon effect in classification tasks, where similarly-accurate models produce conflicting outputs.
One of the main challenges in studying predictive multiplicity is measuring it. 
\citet{semenova2022existence} proposed the Rashomon ratio to measure the Rashomon effect and used a Monte Carlo technique to sample decision tree models for estimation.
\citet{marx2020predictive} quantified predictive multiplicity using optimization formulations to find the worst-case disagreement among all candidate models while controlling for accuracy.
Recently, \citet{hsu2022rashomon,watson2022predictive} proposed other metrics for quantifying predictive multiplicity: Rashomon capacity and viable prediction range.
\citet{black2022model} proposed measures of predictive multiplicity which are applicable to randomized learning.
Our \cref{prop:multi} complements the prior work by providing a closed-form expression for sample complexity of estimating predictive multiplicity which arises due to randomness in training.

\paragraph{Side Effects of Differential Privacy.} To the best of our knowledge, our work is the first one to study the properties of DP training in terms of predictive multiplicity. Multiple works, however, have studied other unintended consequences of private learning. In particular, a number of works \cite{bagdasaryan2019differential, sanyal2022unfair, ganev2022robin} show that DP training comes at a cost of decreased performance for groups which are under-represented in the data. Relatedly, \citet{cummings2019compatibility} show that DP training is incompatible with some notions of algorithmic fairness.

\section{Discussion}

Our theoretical and empirical results show that training with DP and, more broadly,  applying randomization in training increases predictive multiplicity. We demonstrated that higher privacy levels result in higher multiplicity. 
If a training algorithm exhibits high predictive multiplicity for a given input example, the decisions supported by a model's output for this example lose their justifiability: these decisions depend on the randomness used in training rather than on relevant properties or features of this example. 
The connection between privacy in learning and decision arbitrariness might not be obvious to practitioners. 
This lack of awareness is potentially damaging in high-stakes settings (e.g., medical diagnostics, lending, education), where decisions of significant---and potentially life-changing---consequence could be significantly influenced by randomness used to ensure privacy. 

In this concluding section, we discuss whether predictive multiplicity is indeed a valid concern for DP-ensuring algorithms, and outline a path forward.

\subsection{Can the Increase in Predictive Multiplicity be Beneficial?}
Despite the harms of arbitrariness, one might argue that multiplicity can, in some cases, be beneficial.

\paragraph{Opportunities for Satisfying Desirable Properties Beyond Accuracy?} \citet{black2022model} and~\citet{semenova2022existence} argue that multiplicity presents a valuable opportunity. 
In non-private training, the existence of many models that achieve comparable accuracy creates an opportunity for selecting a model which satisfies both an acceptable accuracy level and other useful properties beyond performance, such as fairness \cite{coston2021characterizing}, interpretability \cite{fisher2019all}, or generalizability \cite{semenova2022existence}.
In order to leverage this opportunity, one needs to deliberately steer training towards the model which satisfies  desirable properties beyond accuracy, or search the ``Rashomon set'' of good models \cite{fisher2019all}. 
However, with  randomization alone (e.g., adding Gaussian noise to gradients in training), model designers cannot steer training without compromising DP guarantees, and can only arrive at a model which satisfies additional desirable properties by chance. Thus, this positive side of the multiplicity phenomenon  is not necessarily present in DP-ensuring training. 

It is an open problem to find whether specially-crafted noise distributions or post-processing techniques could be designed to provide the same level of privacy as the standard approaches, and at the same time attain additional useful properties such as fairness.

\paragraph{Predictive Multiplicity is Individually Fair?} 
Individual fairness~\cite{dwork2012fairness} is a formalization of the ``treat like alike'' principle: an individually fair classifier makes similar decisions for individuals who are thought to be similar. 
A way to formally satisfy individual fairness is, in fact, through randomization of decisions. This could lead to an argument that predictive multiplicity is individually fair. For instance, suppose that a predictive model used to assist with hiring decisions is applied to several individuals who are all equally qualified to get the job. Consider two possible decision rules for selecting the candidate to hire with different multiplicity levels. The first rule 
 has high multiplicity: produce a random decision. The second rule has low multiplicity: select a candidate based on lexicographic order. 
As the second decision rule results in a breach of individual fairness and, possibly, a systemic exclusion of some candidates, the first rule with high multiplicity seems preferable.%

This argument, however, only holds if there is randomness \textit{at the prediction stage}. This is not the case for standard DP-ensuring algorithms such as the ones we study. DP training produces one deterministic classifier that is used for all predictions. 
Thus, once training is done, there is no randomization of decisions as in the example above. Thus, the decisions due to such DP-ensuring models are no different than arbitrary rules such as selection based on lexicographic ordering.

\paragraph{Overcoming the Algorithmic Leviathan?} 
\citet{creel2022algorithmic} consider a setting where different decision-making systems which have high impact on an individual's livelihood, e.g., credit scoring systems from competing bureaus in the USA~\cite{citron2014scored}, are trained in ways that lead to all of them outputting the same decisions. 
This \textit{algorithmic monoculture} would completely remove the possibility of accessing resources for some individuals, as turning to a competing decision-maker would not change the outcome. 
In this case, \citeauthor{creel2022algorithmic} argue that high predictive multiplicity could be a desirable property as it enables to access resources across the decision-makers. 

In some high-stakes settings, such as healthcare, an algorithmic monoculture might \emph{not} pose a concern. Indeed, one would wish that predictive models used as a part of a diagnostic procedure for a disease output a consistent decision so that patients can be treated (or not treated) as needed. In this scenario, in fact, predictive multiplicity could potentially harm patients by either delaying a patient's treatment, or recommending a treatment when the patient is healthy. 
In such settings, the positive impact of predictive multiplicity in avoiding an algorithmic Leviathan loses meaning. 

Regardless of whether algorithmic monoculture is a legitimate concern or not for a given application, it is helpful for model designers and decision subjects to be informed of the level of predictive multiplicity, whether to gauge the likelihood of recourse, or brace for the arbitrariness of decisions.

\subsection{Open Problems}
\label{sec:open}

\paragraph{Reporting Mutiplicity}
Potential mitigations of the harms of predictive multiplicity could be to abstain from outputting a prediction with high multiplicity, or to communicate the magnitude of  multiplicity to the stakeholders. Doing so is challenging: any sort of communication of disagreement values could partially reveal information about the privacy-sensitive training data and break DP guarantees. Consider, as before, the setting of using a predictive model to assist in a medical diagnosis.
Whether a model abstains from predictions or outputs them along with disagreement estimates, there is a certain amount of information leakage about the training data to doctors. If the disagreement estimates are computed on privacy-sensitive data and are used without appropriate privatization---whether published or used to decide on abstention---they can reveal information about the data. 
To address this issue, one could use privacy-preserving technologies such as DP to abstain from making a prediction based on a high disagreement value or report the disagreement estimate in a privacy-preserving way. 
Studying whether effective privatization of disagreement computations is possible is an open problem for future work.

\paragraph{General Characterization of the Predictive-Multiplicity Costs of DP} 
We have theoretically characterized the predictive-multiplicity behavior of the output-perturbation mechanism as applied to logistic regression. Doing so for other mechanisms and model families is a non-trivial undertaking. In this work, we resort to empirical measurement with re-training. An open problem is finding whether we can characterize these behaviors for a wider range of model families, mechanisms, or even for any general mechanism which satisfies DP.

\subsection{Recommendations Moving Forward}
As discussed in the previous sections, existing techniques do not enable model designers to eliminate, or even mitigate, the implications of predictive multiplicity when using DP-ensuring models.
We have pointed out which open problems would need to be solved in order to reduce the impact of predictive multiplicity in high-stakes privacy-sensitive scenarios. Until DP mechanisms that mitigate multiplicity become available, the negative effects of multiplicity can only be countered by \textit{auditing for multiplicity} prior to deployment. Therefore, in order to understand the impact of privacy on the justifiability of model decisions, model designers should directly measure predictive multiplicity when using DP training, e.g., using the methods we introduce in \cref{sec:estimator}. %
If at the desired level of privacy the training algorithm exhibits high predictive multiplicity (either in general or for certain populations), model designers should carefully consider whether the use of such models is justified in the first place.

\begin{acks}
This work is partially funded by the Swiss National Science Foundation under grant 200021-188824, and the US National Science Foundation under grants CAREER 1845852, FAI 2040880, and CIF 1900750. Hsiang Hsu acknowledges support from Meta Ph.D. Fellowship. The authors would like to thank Salil Vadhan, Borja Balle, Jakab Tardos, and the anonymous reviewers at FAccT 2023 for their helpful feedback.
\end{acks}

\bibliographystyle{ACM-Reference-Format}
\bibliography{main}

\appendix

\input{appendix}

\end{document}

%% file: appendix.tex
\section{Omitted Proofs and Derivations}
\label{app:proofs}

\subsection{\cref{sec:background}}
\label{app:background}

First, we provide an explanation on the range of disagreement without normalization:
\begin{proposition}[Range of non-normalized disagreement]
The expression $\Pr[f_{\theta}(x) \neq f_{\theta'}(x)]$ has range of $[0, 0.5]$.
\end{proposition}
\begin{proof}
As $f_\theta(x) \in \{0, 1\}$, we can assume $\Pr[f_\theta(x)=1]=p$, and thus $\Pr[f_{\theta}(x) \neq f_{\theta'}(x)] = \Pr[f_{\theta}(x)=0\; \text{and}\; f_{\theta'}(x)=1] + \Pr[f_{\theta}(x)=1\; \text{and}\; f_{\theta'}(x)=0] = 2p(1-p) \in [0,0.5]$.
\end{proof}

Next, we provide a proof that disagreement is proportional to variance in our setup:
\begin{proof}[Proof of Proposition \ref{stmt:mult-var}]
As $f_\theta(x) \in \{0, 1\}$, we have that
\begin{equation}
\begin{aligned}
    \predmult(x) &= 2 \Pr_{\theta, \theta' \sim \rashomon}[f_{\theta}(x) \neq f_{\theta'}(x)] \\
    &= 2 \, \E_{\theta, \theta' \sim \rashomon}[\id[f_{\theta}(x) \neq f_{\theta'}(x)]] \\
    &= 2 \, \E_{\theta, \theta' \sim \rashomon}[(f_\theta(x) - f_{\theta'}(x))^2] \\
    &= 2 \, \E_{\theta \sim \rashomon}[f_{\theta}^2(x)] - 4 \, \E_{\theta, \theta' \sim \rashomon}[ f_{\theta}(x) \cdot f_{\theta'}(x)] \\
    & \quad + 2 \E_{\theta' \sim \rashomon}[f_{\theta'}^2(x)] \\
    &= 4 \, (\E_{\theta \sim \rashomon}[ f_{\theta}(x)]^2 \\
    & \quad - \E_{\theta \sim \rashomon}[f_{\theta}(x)] \cdot \E_{\theta' \sim \rashomon}[f_{\theta'}(x)]) \\
    &= 4 \, \mathrm{Var}_{\theta \sim \rashomon}(f_\theta(x))\\
    &= 4p_x(1-p_x),
\end{aligned}
\end{equation}
where $p_x(1-p_x)$ is the population variance of the r.v. $f_\theta(x) \sim \mathrm{Bernoulli}(p_x)$.
\end{proof}

\subsection{\cref{sec:out-pert}}

\begin{proof}[Proof of Proposition \ref{prop:outpert}]
First, observe that the expression
\[
    p_x = \E_{\theta_\mathsf{priv} \sim \rashomon}[f_{\theta_\mathsf{priv}}(x)]
\]
can be expressed as:
\begin{equation}\label{eq:px}
    \begin{aligned}
        \E[f_{\theta_\mathsf{priv}}(x)] &= \E[\id[ \mathsf{sigmoid}(\theta_\mathsf{priv}^\intercal x) > 0.5 ]] \\
        &= \E[\id[ \theta_\mathsf{priv}^\intercal x > 0 ]] \\
        &= \Pr( \theta_\mathsf{priv}^\intercal x > 0 ).
    \end{aligned}
\end{equation}

Denoting by $\xi \define \mathcal{N}(0, 1)$ and $\xi_d \define \mathcal{N}(0, I_d)$, we can see that the score $\theta_\mathsf{priv}^\intercal \, x$ is equal to:
\begin{equation}\label{eq:closed-noise}
    \begin{aligned}
    \theta_\mathsf{priv}^\intercal \, x &= (\theta_\mathsf{np} + \sigma \xi_d)^\intercal x \\
    &= \theta_\mathsf{np}^\intercal \, x + \sigma \sum_{i=1}^d x_i \xi \\
    &= \theta_\mathsf{np}^\intercal \, x + \sqrt{\sum_{i=1}^d x_i^2} \cdot \sigma\xi \\
    &= \theta_\mathsf{np}^\intercal \, x + \|x\| \sigma \xi.
    \end{aligned}
\end{equation}

Plugging in the closed form in \cref{eq:closed-noise} into \cref{eq:px}, we get:
\begin{equation}
    p_x = \Pr(\theta_\mathsf{np}^\intercal \, x + \|x\| \sigma \xi > 0) = \Pr\left(\xi > -\frac{\theta_\mathsf{np}^\intercal \, x}{\|x\| \cdot \sigma}\right) = \Phi\left( \frac{\theta_\mathsf{np}^\intercal \, x}{\|x\| \cdot \sigma} \right).
\end{equation}

\end{proof}

\subsection{\cref{sec:estimator}}
\label{app:estimator}

\begin{proof}[Proof of Proposition \ref{prop:unbiasedest}]
The $\nicefrac{1}{m-1}$ term comes from Bessel's correction. Observe that 
\begin{equation}
\begin{aligned}
\E\left[\frac{m}{m-1} \, \hat p_x(1 - \hat p_x)\right] &= \frac{m}{m-1} (\E[\hat p_x]-\E[\hat p_x^2]) \\
&= \frac{m}{m-1} (\E[\hat p_x]-\mathrm{Var}(\hat p_x) - \E[\hat p_x]^2) \\
&= \frac{m}{m-1}\left(p_x-\frac{p_x(1-p_x)}{m}-p_x^2\right)  \\
&= p_x(1-p_x)
\end{aligned}
\end{equation}
Therefore, $\E[\hat \predmult(x)] = 4p_x(1-p_x) = \predmult(x)$.
\end{proof}

\begin{proof}[Proof of Proposition \ref{prop:single}]
As $\hat \predmult(x)$ is a continuous transformation of $\hat p_x$, we could bound  the deviation $|\hat \predmult(x)-\predmult(x)|$ by $|\hat p_x-p_x|$. 
Suppose $\hat p_x = p_x + \nu$ and $\nu \in [-\eta,\eta]$, we have
\begin{equation}
\begin{aligned}\label{eq:variance-estimation-1}
    &\left|\frac{m}{m-1} \, \hat p_x(1 - \hat p_x) - p_x(1-p_x)\right| =\\
    &=\left|\frac{m}{m-1} \, (p_x + \nu)(1 - p_x - \nu) - p_x(1-p_x)\right|\\
    &= \left|\left(\frac{m}{m-1}-1\right)p_x(1-p_x) + \frac{m}{m-1} \nu(1 - 2p_x - \nu)\right|\\
    &\leq \frac{p_x(1-p_x)}{m-1} + \frac{m}{m-1}|\nu||1-2p_x+\nu|\\
    &\leq \frac{p_x(1-p_x)}{m-1} + \frac{m}{m-1}|\nu|(1+|\nu|)\\
    &\leq \frac{p_x(1-p_x)}{m-1} + \frac{m}{m-1}\eta(1+\eta)\\
    &\leq \frac{1}{4(m-1)} + \frac{m}{m-1}\eta(1+\eta).
\end{aligned}
\end{equation}
By Chernoff-Hoeffding inequality, we have the following concentration bounds on the sample mean $\hat p_x$,
\begin{equation}\label{eq:variance-estimation-2}
    \Pr[ |\hat p_x - p_x| \geq \nu ] \leq 2 \exp\left(-2 \nu^2 m\right).
\end{equation}
Thus with probability at least $1 - \rho$, we have:
\[|
    \hat p_x - p_x| \leq \sqrt{\log (2/\rho)/2m}.
\]
Combining \cref{eq:variance-estimation-1} and \cref{eq:variance-estimation-2}, we have
\begin{equation}\label{eq:variance-estimation-3}
\begin{aligned}
    |\hat \predmult(x) - \predmult(x)| &= \left| \frac{4m}{m-1} \, \hat p_x(1 - \hat p_x) - 4p_x(1-p_x) \right| \\
    &\leq \frac{1}{(m-1)} + \frac{4m}{m-1}\eta(1+\eta).
\end{aligned}
\end{equation}
Plugging $\eta = \sqrt{\log (2/\rho)/2m}$ into Eq.~\eqref{eq:variance-estimation-3} yields the desired result.
Note that by solving $\frac{1}{(m-1)} + \frac{4m}{m-1}\eta(1+\eta)\leq \alpha$ with $\eta = \sqrt{\log (2/\rho)/2m}$ with conditions $\alpha > 0$ and $0 < \rho < 1$, we have:
\begin{equation}
    m \geq 1 + \frac{\alpha + 2t(2+\alpha)+2\sqrt{2}\sqrt{t(1+\alpha)(2t+\alpha)}}{\alpha^2},
\end{equation}
where $t = \log(2/\rho).$
\end{proof}

\begin{proof}[Proof of Proposition \ref{prop:multi}]
Since the samples are i.i.d., we have the following union bound for the concentration of sample mean,
\begin{equation}
\begin{aligned}
    \Pr\left[ \bigcup_{i=1}^k \{ |\hat p_{x_i} - p_{x_i}| \geq \nu \} \right] & \leq \prod_{i=1}^k \Pr[ |\hat p_{x_i} - p_{x_i}| \geq \nu] \\
    & \leq 2k \exp\left(-2 \nu^2 m\right).
\end{aligned}
\end{equation}
Therefore, with probability $1-\rho$, $|\hat p_{x_i} - p_{x_i}| \leq \sqrt{\log (2k/\rho)/2m}$ for $i = 1,\ldots,k$, and the desired result follows the derivation in Proposition~\ref{prop:single}.

\end{proof}

\section{Additional Experimental Details}
\label{sec:exp-details}

\subsection{Details on the Experimental Setup}

\paragraph{Datasets} For illustrative purposes, we use the following classes as our target labels. For the Credit dataset, we use ``Approved'' as the target label. For the Contraception dataset, we use ``long-term method''. For the dermatology dataset, we use ``seboreic dermatitis'' diagnosis. For the Mammography dataset, we use ``malignant''.

\begin{table}[]
\centering
\caption{Summary of datasets used in our experimental evaluations.}
\label{tab:datasets}
\resizebox{.9\linewidth}{!}{
\begin{tabular}{l|rc|rr}
\toprule
Dataset       & Size  & Number of features      & Train size & Test size   \\
\midrule
Synthetic     & $\infty$ & 2                       & 2000      & 20,000   \\  
Credit        & 653    & 46                      & 489         & 164       \\
Contraception & 1,473  & 9                       & 1,104       & 369       \\
Mammography   & 830    & 5                       & 622         & 208       \\
Dermatology   & 358    & 34                      & 268         & 90        \\
CIFAR-10      & 60,000 & $32 \times 32 \times 3$ & 50,000      & 10,000    \\
\bottomrule
\end{tabular}
}
\end{table} 

\paragraph{CIFAR-10} We use the convolutional neural network trained over the ScatterNet features~\cite{oyallon2015deep} following \citet[Table 9, Appendix]{tramer2021differentially}. We use DP-SGD with batch size of 2048, learning rate of 4, Nesterov momentum of 0.9, and gradient clipping norm of 0.1. We vary the gradient noise multiplier $\sigma$ to achieve the privacy levels of $\varepsilon \approx 2.22, 2.73, 3.62. 4.39, 5.59$ as computed by the Moments accountant~\cite{abadi2016deep}.

\paragraph{Software} We use the following software:
\begin{itemize}
    \item diffprivlib~\cite{diffprivlib} for the implementation of objective-perturbation for logistic regression.
    \item PyTorch~\cite{pytorch} for implementing neural networks.
    \item opacus~\cite{opacus} for training PyTorch neural networks with DP-SGD.
    \item numpy~\cite{numpy}, scipy~\cite{scipy}, and pandas~\cite{pandas} for numeric analyses.
    \item seaborn~\cite{seaborn} for visualizations.
\end{itemize}

\subsection{Multiplicity of Predictions vs. Scores}
\label{sec:vp}
Recall that the models we consider are not only capable of outputting a binary prediction but also a confidence score. The disagreement metric in \cref{def:disagreement}, however, only uses the predictions after applying a threshold. To verify if the trends we observe persist also at the level of confidence scores, we additionally evaluate \newterm{viable prediction range}, a metric for measuring multiplicity of the confidence scores proposed by \citet{watson2022predictive}:
\begin{equation}
    \predmult_\mathsf{vp}(x) \define \max_{\theta \sim \rashomon} h_\theta(x) - \min_{\theta \sim \rashomon} h_\theta(x)
\end{equation}

\cref{fig:syn-vp} shows the viable prediction range for different values in the input space for logistic regression trained with objective perturbation on our synthetic dataset. The regions with high viable prediction range overlap with the regions with high disagreement (see \cref{fig:syn-demo}). This is also consistent with the results on the tabular datasets, for which \cref{fig:tabular-vp} shows both disagreement and viable prediction range increasing on average as the level of privacy increases. 

\vspace{1mm}\noindent\textbf{Implications.} Models trained with a high level of privacy exhibit high multiplicity both of their confidence scores (in terms of viable prediction range) and of ``hard'' predictions after applying a threshold (in terms of disagreement).

\begin{figure*}
    \includegraphics[width=0.9\linewidth]{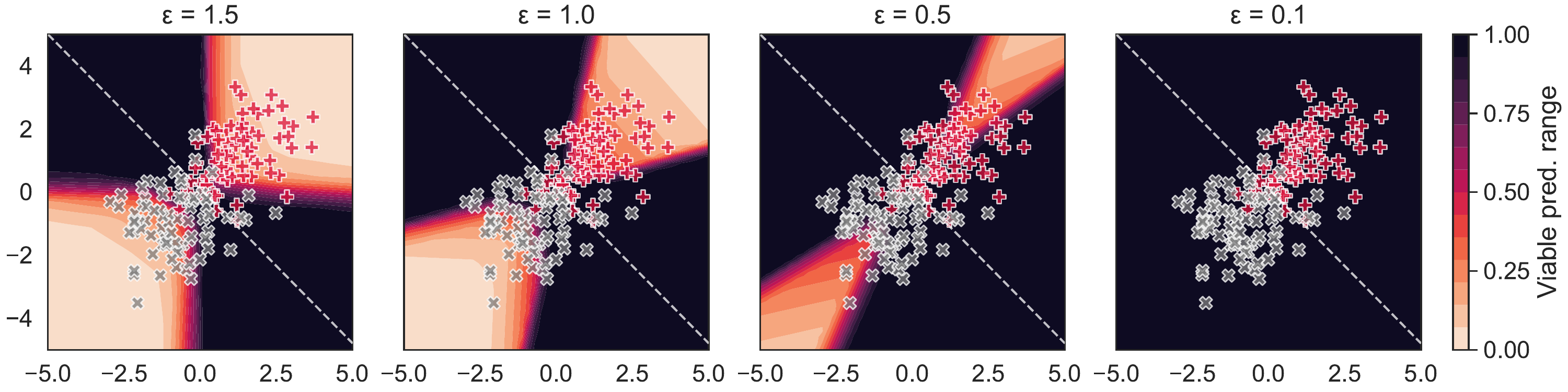}
    \caption{ Viable prediction range of logistic regression trained with objective perturbation is high for examples for which disagreement is also high. See \cref{fig:syn-demo} for the disagreement values and details of the plot setup.}
    \label{fig:syn-vp}
\end{figure*}

\begin{figure*}
    \includegraphics[width=0.7\linewidth]{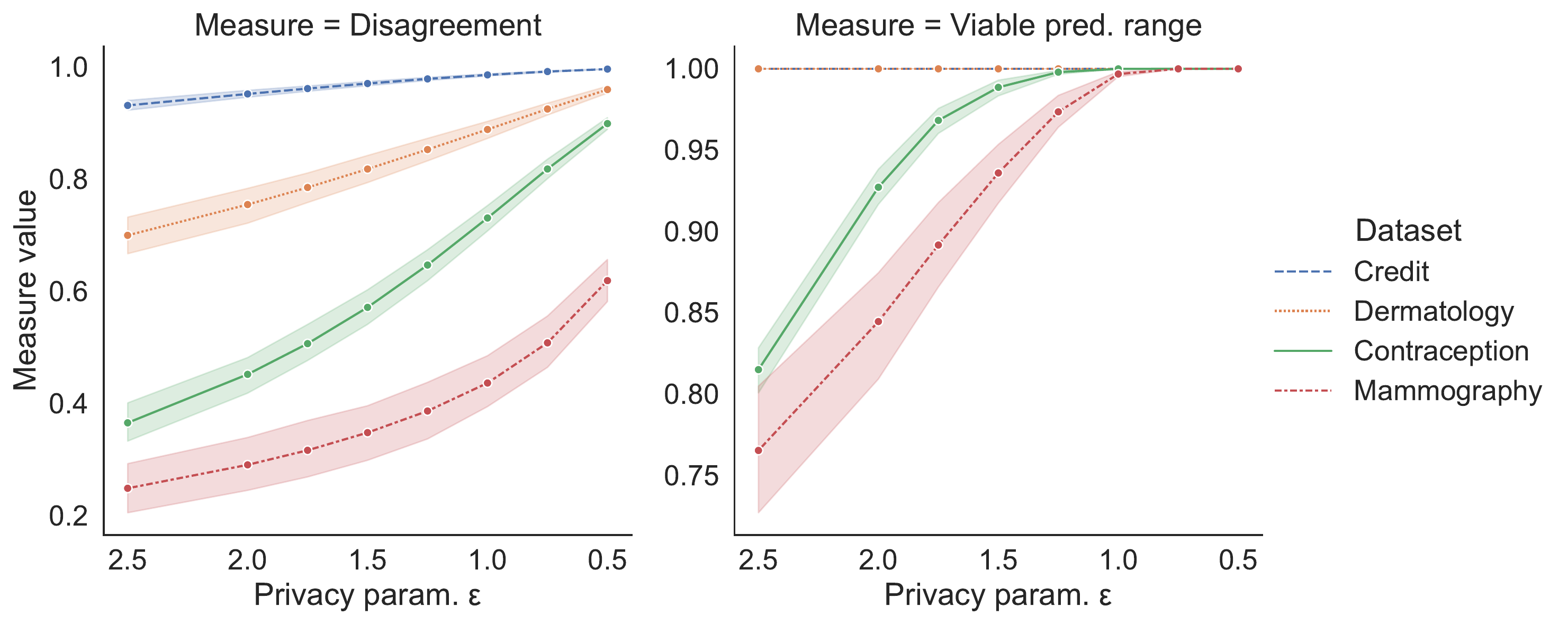}
    \caption{ Both the disagreement and viable prediction range of logistic regression trained with objective perturbation on tabular datasets increases as the level of privacy increases. See \cref{fig:mult-main} for the details of the plot setup.}
    \label{fig:tabular-vp}
\end{figure*}

\subsection{Additional Figures and Tables}\label{app:extras}
The rest of the document contains additional figures and tables.

\begin{figure*}
    \centering
    \begin{subfigure}[t]{.45\linewidth}
    \includegraphics[width=\linewidth]{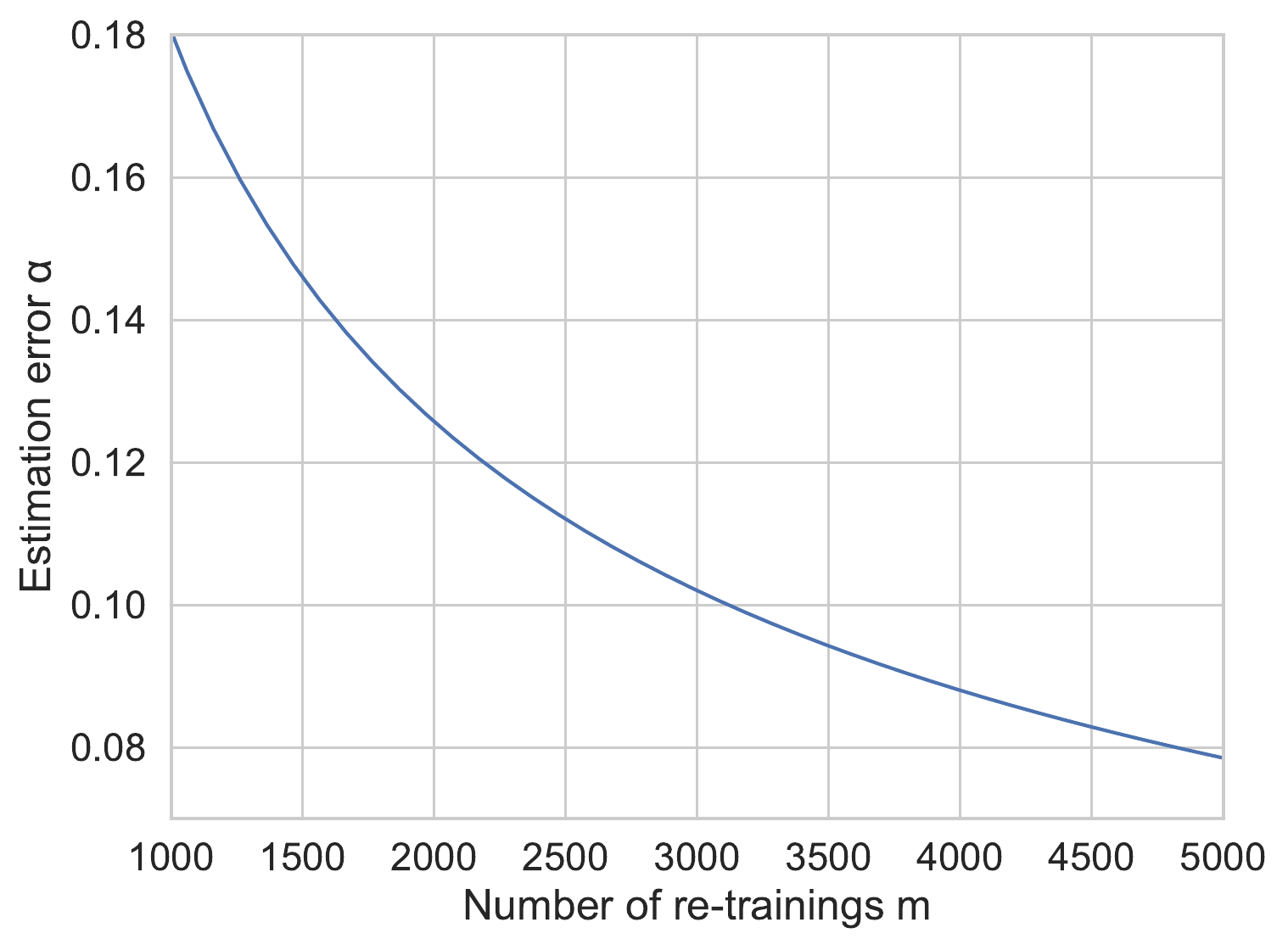}
    \caption{Theoretical error of estimating disagreement w.p. 95\%}
    \label{fig:estimator-rate}
    \end{subfigure}
    ~
    \begin{subfigure}[t]{.45\linewidth}
    \includegraphics[width=\linewidth]{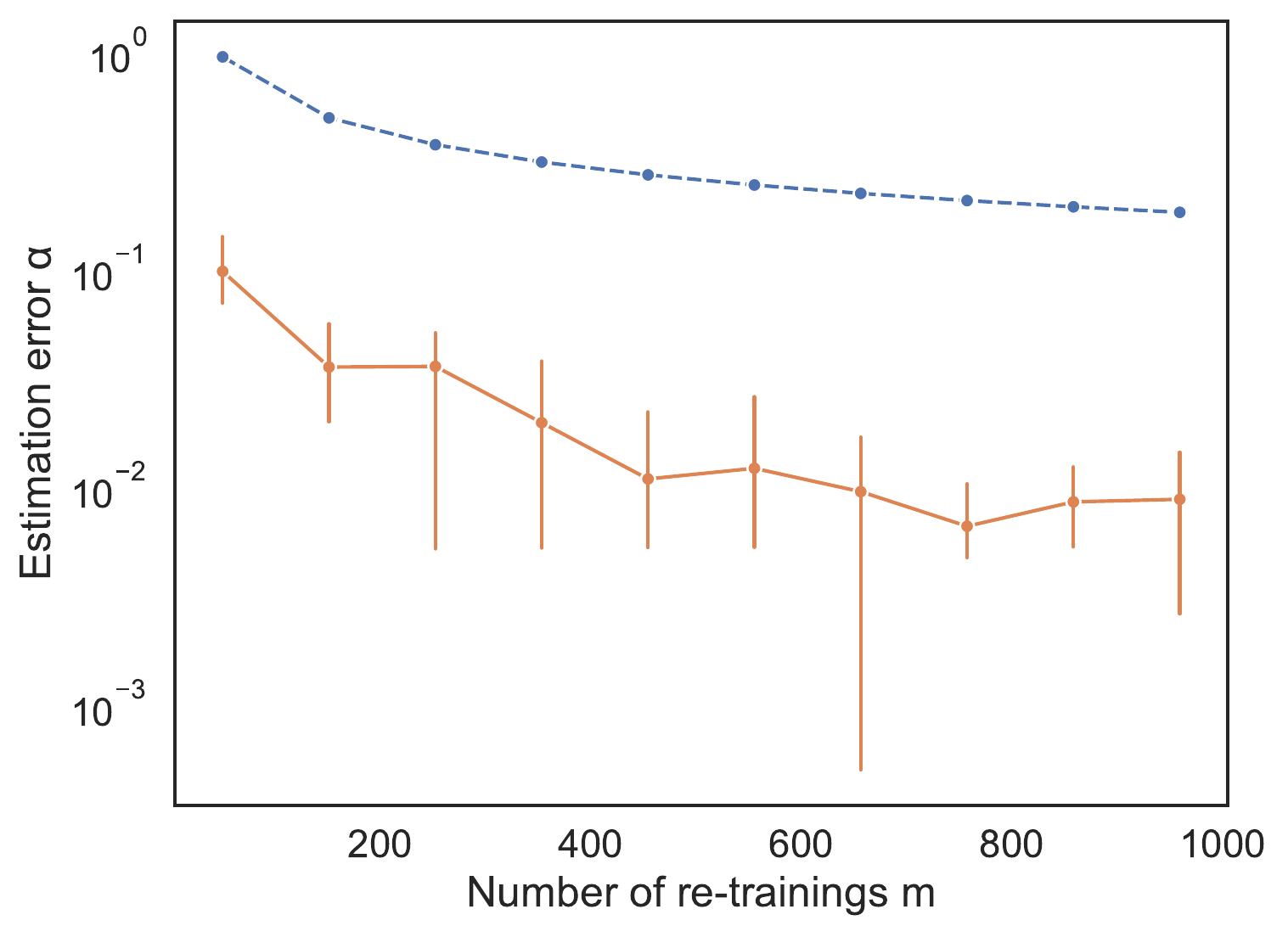}
    \caption{Empirical error of estimating disagreement for one arbitrarily chosen example (solid orange line \textcolor{orange}{---}) compared to the theoretical maximum error w.p. 95\% (dashed blue line \textcolor{blue}{$-\,-$}). The error bars are 95\% confidence intervals over 10 re-samplings of $m$ models. This suggests that the theoretical upper bound on error is pessimistic in practice. y axis is logarithmic.}    
    \label{fig:estimator-rate-vs-real-error}
    \end{subfigure}

    \caption{Visualization of disagreement estimation error as a function of the number of models sampled from the training distribution $\rashomon$.}
\end{figure*}

\begin{table*}
\caption{ Summary statistics of the performance and predictive-multiplicity measures on real-world datasets. For tabular datasets, the performance metrics are the area under the ROC curve (AUC), and the harmonic mean of precision and recall ($F_1$ score) on the test data. For CIFAR-10, the performance metric is the accuracy on the test data. For these, we report mean and standard deviation over the $m$ re-trained models. For disagreement, we report mean, standard deviation, minimum, median, maximum, the 90-th percentile, and the 95-th percentile over the examples in each respective test dataset. Observe that for every dataset there exist multiple examples with high level of predictive multiplicity even if the average level of predictive multiplicity for the given dataset is low. E.g., compare the 95-th percentile of disagreement on the CIFAR-10 dataset at $\varepsilon = 2.22$ (0.81) to its mean value (0.11).}
\label{tab:exp-stats}

\centering
\begin{tabular}{ll|rr|rr|rrrrrrr}
\toprule
      &  & \multicolumn{2}{c|}{AUC} & \multicolumn{2}{c|}{$F_1$ score} & \multicolumn{7}{c}{Disagreement} \\
      Dataset  & $\varepsilon$ &   Mean &    Std. &     Mean &    Std. &     Mean &      Std. &      Min &   Median &      Max &   90 pctl. &   95 pctl. \\
\midrule
Contraception & 0.50 &  57.51 &   6.72 &  48.72 &   7.86 &     0.90 &  0.10 &  0.48 &   0.93 &  1.00 &   0.99 &   1.00 \\
      & 0.75 &  60.26 &   6.20 &  50.29 &   7.54 &     0.82 &  0.17 &  0.24 &   0.88 &  1.00 &   0.99 &   1.00 \\
      & 1.00 &  62.50 &   5.47 &  51.56 &   7.09 &     0.73 &  0.23 &  0.11 &   0.79 &  1.00 &   0.98 &   1.00 \\
      & 1.25 &  64.27 &   4.71 &  52.62 &   6.62 &     0.65 &  0.27 &  0.05 &   0.70 &  1.00 &   0.97 &   0.99 \\
      & 1.50 &  65.62 &   4.00 &  53.53 &   6.14 &     0.57 &  0.30 &  0.02 &   0.60 &  1.00 &   0.96 &   0.99 \\
      & 1.75 &  66.65 &   3.38 &  54.31 &   5.67 &     0.51 &  0.32 &  0.00 &   0.50 &  1.00 &   0.95 &   0.99 \\
      & 2.00 &  67.43 &   2.86 &  54.98 &   5.21 &     0.45 &  0.33 &  0.00 &   0.42 &  1.00 &   0.94 &   0.98 \\
      & 2.50 &  68.49 &   2.10 &  55.97 &   4.39 &     0.37 &  0.33 &  0.00 &   0.27 &  1.00 &   0.92 &   0.97 \\
\midrule
Credit & 0.50 &  52.22 &  15.95 &  46.48 &  16.38 &     1.00 &  0.00 &  0.99 &   1.00 &  1.00 &   1.00 &   1.00 \\
      & 0.75 &  53.72 &  15.70 &  47.84 &  15.70 &     0.99 &  0.01 &  0.98 &   0.99 &  1.00 &   1.00 &   1.00 \\
      & 1.00 &  55.16 &  15.41 &  49.15 &  15.05 &     0.99 &  0.01 &  0.96 &   0.99 &  1.00 &   1.00 &   1.00 \\
      & 1.25 &  56.56 &  15.06 &  50.39 &  14.46 &     0.98 &  0.02 &  0.94 &   0.98 &  1.00 &   1.00 &   1.00 \\
      & 1.50 &  57.86 &  14.69 &  51.59 &  13.89 &     0.97 &  0.03 &  0.91 &   0.98 &  1.00 &   1.00 &   1.00 \\
      & 1.75 &  59.10 &  14.31 &  52.72 &  13.35 &     0.96 &  0.03 &  0.89 &   0.97 &  1.00 &   1.00 &   1.00 \\
      & 2.00 &  60.26 &  13.91 &  53.77 &  12.85 &     0.95 &  0.04 &  0.86 &   0.96 &  1.00 &   1.00 &   1.00 \\
      & 2.50 &  62.41 &  13.12 &  55.70 &  12.05 &     0.93 &  0.06 &  0.80 &   0.95 &  1.00 &   1.00 &   1.00 \\
\midrule
Dermatology & 0.50 &  62.19 &  19.76 &  48.81 &  17.88 &     0.96 &  0.03 &  0.89 &   0.96 &  1.00 &   1.00 &   1.00 \\
      & 0.75 &  66.75 &  17.65 &  52.67 &  16.44 &     0.93 &  0.05 &  0.79 &   0.93 &  1.00 &   0.99 &   0.99 \\
      & 1.00 &  70.44 &  15.83 &  55.88 &  15.21 &     0.89 &  0.08 &  0.69 &   0.90 &  1.00 &   0.98 &   0.99 \\
      & 1.25 &  73.46 &  14.28 &  58.57 &  14.20 &     0.85 &  0.10 &  0.60 &   0.86 &  1.00 &   0.98 &   0.98 \\
      & 1.50 &  75.94 &  12.97 &  60.93 &  13.30 &     0.82 &  0.12 &  0.52 &   0.83 &  1.00 &   0.97 &   0.98 \\
      & 1.75 &  78.04 &  11.89 &  62.98 &  12.60 &     0.79 &  0.13 &  0.46 &   0.80 &  1.00 &   0.95 &   0.97 \\
      & 2.00 &  79.80 &  10.96 &  64.78 &  12.00 &     0.75 &  0.15 &  0.39 &   0.77 &  0.99 &   0.94 &   0.96 \\
      & 2.50 &  82.66 &   9.45 &  67.80 &  10.95 &     0.70 &  0.17 &  0.32 &   0.72 &  0.99 &   0.92 &   0.94 \\
\midrule
Mammography & 0.50 &  75.64 &   8.95 &  69.22 &   9.88 &     0.62 &  0.28 &  0.20 &   0.61 &  1.00 &   0.98 &   1.00 \\
      & 0.75 &  78.57 &   6.51 &  72.46 &   7.04 &     0.51 &  0.34 &  0.07 &   0.45 &  1.00 &   0.98 &   1.00 \\
      & 1.00 &  80.36 &   5.26 &  74.39 &   5.48 &     0.44 &  0.36 &  0.02 &   0.33 &  1.00 &   0.97 &   0.99 \\
      & 1.25 &  81.62 &   4.44 &  75.64 &   4.66 &     0.39 &  0.37 &  0.01 &   0.24 &  1.00 &   0.95 &   0.99 \\
      & 1.50 &  82.54 &   3.82 &  76.56 &   4.14 &     0.35 &  0.37 &  0.00 &   0.17 &  1.00 &   0.93 &   0.99 \\
      & 1.75 &  83.25 &   3.36 &  77.29 &   3.81 &     0.32 &  0.36 &  0.00 &   0.12 &  1.00 &   0.91 &   0.98 \\
      & 2.00 &  83.81 &   2.98 &  77.85 &   3.56 &     0.29 &  0.35 &  0.00 &   0.08 &  1.00 &   0.89 &   0.98 \\
      & 2.50 &  84.61 &   2.40 &  78.70 &   3.22 &     0.25 &  0.34 &  0.00 &   0.04 &  1.00 &   0.84 &   0.96 \\
\bottomrule
\end{tabular}
\\[2em]
\begin{tabular}{ll|rr|rrrrrrr}
\toprule
        & & \multicolumn{2}{c|}{Accuracy} & \multicolumn{7}{c}{Avg. Disagreement across Classes} \\
        Dataset  & $\varepsilon$ &   Mean &    Std. &     Mean &      Std. &      Min &   Median &      Max &   90 pctl. &   95 pctl. \\
\midrule
CIFAR-10 & 2.22 &    65.38 &     0.32 &          0.11 &          0.25 &           0.0 &           0.0 &           1.0 &          0.48 &          0.81 \\
        & 2.73 &    67.65 &     0.35 &          0.09 &          0.23 &           0.0 &           0.0 &           1.0 &          0.36 &          0.77 \\
        & 3.62 &    69.56 &     0.32 &          0.08 &          0.22 &           0.0 &           0.0 &           1.0 &          0.29 &          0.69 \\
        & 4.39 &    70.38 &     0.33 &          0.07 &          0.21 &           0.0 &           0.0 &           1.0 &          0.23 &          0.64 \\
        & 5.59 &    71.06 &     0.29 &          0.06 &          0.20 &           0.0 &           0.0 &           1.0 &          0.15 &          0.59 \\
\bottomrule
\end{tabular}

\end{table*}